\documentclass[12pt]{article}

\usepackage{a4wide}
\usepackage{titlesec}
\usepackage[center, font=small]{caption}
\usepackage{amsmath}
\usepackage{amssymb}
\usepackage{authblk}
\usepackage{setspace}
\usepackage{floatrow}
\usepackage{placeins}
\usepackage{listings}
\usepackage{subcaption}
\usepackage[hyphens]{url}
\usepackage{natbib}
\usepackage{rotating}
\usepackage{threeparttable}
\usepackage{array}
\usepackage{multirow}
\usepackage{appendix}
\usepackage{longtable}
\usepackage{todonotes}
\usepackage[multiple]{footmisc}
\usepackage[utf8]{inputenc}
\usepackage[T1]{fontenc}
\usepackage{xcolor}

\definecolor{perplexityteal}{HTML}{20808D}
\usepackage[colorlinks=true,
            linkcolor=perplexityteal,
            citecolor=perplexityteal,
            urlcolor=perplexityteal]{hyperref}
\usepackage{booktabs}
\usepackage{nicefrac}
\usepackage{graphicx}
\usepackage{newtxtext}
\usepackage{amsthm}
\newtheorem{prop}{Proposition}
\newtheorem{lem}{Lemma}
\newtheorem{cor}{Corollary}
\newcounter{assumption}
\newcommand{\asmplabel}[1]{\refstepcounter{assumption}\label{#1}}
\newcommand{\Conversational}{\text{Conversational}}
\newcommand{\Agent}{\text{Agent}}

\newcounter{en}
\title{\normalfont How AI Agents Reshape Knowledge Work:
\\Autonomy, Efficiency, and Scope\thanks{Correspondence to Jeremy Yang (\href{mailto:jeryang@hbs.edu}{\texttt{jeryang@hbs.edu}}) and Jerry Ma (\href{mailto:jerry@perplexity.ai}{\texttt{jerry@perplexity.ai}}).}
}

\makeatletter
\def\@maketitle{%
  \newpage
  \begin{center}%
    {\LARGE \@title \par}%
    \vskip 1.5em%
    {
      Jeremy Yang\textsuperscript{1} \
      Kate Zyskowski\textsuperscript{2} \
      Noah Yonack\textsuperscript{2} \
      Jerry Ma\textsuperscript{2}\\[1em]
      \textsuperscript{1}Harvard University \quad \\
      \vskip 0.5em
      \textsuperscript{2}Perplexity \\
    }%
    \vskip 1.5em
    June 8, 2026
  \end{center}%
  \par
  \vskip 1.5em}
\makeatother

\begin{document}
\maketitle

\begin{abstract}
Frontier AI systems are bridging the gap between intelligence and utility by shifting from conversational assistants to autonomous agents that execute tasks end to end. Using data from Perplexity's Search and Computer products, we study this transition by examining how AI agents accelerate and reshape knowledge work. We adopt an individual-level task-based framework where agents have a higher fixed delegation cost but a lower marginal execution cost per step. This framework predicts that agent access expands the affordable task frontier toward weakly higher-value tasks and weakly increases realized value; when the pre-agent budget binds, surplus and the value-to-cost ratio also weakly increase. Turning to the data, we document three key empirical findings. First, using matched session pairs with near-identical initial queries as natural experiments for the same underlying task attempted with both products, Computer performs 26~minutes of autonomous work per user session, versus 33~seconds for Search. Computer automates task decomposition and execution that Search users might otherwise manually orchestrate and implement. As a result, Computer shifts the follow-up query distribution toward higher-order work such as verification and extension. Autonomy also increases execution quality, with per-query medium-to-high dissatisfaction rates 55\% lower on Computer than on Search. Second, due to its autonomy advantage, Computer reduces completion time from 269 to 36~minutes on matched tasks, lowering estimated time and cost by 87\% and 94\%, respectively, compared to humans equipped with Search alone. Third, Computer changes the scope of work that users attempt: Computer queries more often cross occupational boundaries, require higher-order cognition, draw on broader expertise, take the form of composite tasks that bundle multiple subtasks into a single query, and unlock work activities that are essentially absent from Search usage among the same users. Together, the evidence indicates that AI agents accelerate workflows, enhance output quality, reduce costs, and expand the breadth and depth of automated work.

\end{abstract}

\clearpage
\section{Introduction}

A central question in the economics of AI is how it reshapes knowledge work \citep{agrawal2026economics}. AI and user behavior are co-evolving, producing a shifting landscape of what AI can do, how AI is used, and what its downstream economic impact is. As capabilities advance, AI products are closing the gap between intelligence and utility, changing how they are integrated into real-world workflows and creating new sources of value and new structures of work. 

Over the last few years, frontier products have progressed from conversational assistants to copilots to agents. Conversational assistants (e.g., chatbots) primarily support isolated information exchange with limited context or ability to act. Copilots embed these capabilities into existing tools and workflows, co-working with users to complete tasks within those tools' interfaces. Agents go further: they connect across a wider range of tools in the backend and return completed artifacts with little human involvement. The shift is from AI as a conversational assistant to AI as an end-to-end work execution engine, characterized by greater autonomy and deeper integration into the user's entire digital environment.

We use data from Perplexity to study the implications of this transition by comparing how knowledge work is completed with conversational assistants versus agents. As background, Figure~\ref{fig:product_progression} situates Perplexity's product portfolio within a two-dimensional space of autonomy and context. Autonomy captures the extent to which a system can plan and execute actions on behalf of the user with minimal human intervention. Context integration captures the extent to which the system can read from and write to the user's environment, including external tools and connected services. We use three Perplexity products to illustrate the broader landscape:
\begin{itemize}
\item \textbf{Perplexity Search} represents the baseline. Released in 2022, Perplexity Search introduced the \emph{answer engine} product category: it allows a user to ask a question and receive a cited, synthesized answer from a knowledge base comprising billions of documents.

\item \textbf{Comet Assistant} represents an advance in both autonomy and context. In 2025, Perplexity released the Comet web browser. Its flagship feature, Comet Assistant, is an agent that helps users access knowledge and perform work within their browser. Comet Assistant makes human-AI integration more continuous by moving interactions into the application layer, where much knowledge work already occurs, allowing AI to co-work with the user by reasoning over and acting upon open-world web environments.
\item \textbf{Perplexity Computer} advances even further across both autonomy and context. Released in 2026, Computer is a general-purpose agent orchestration system that performs work across increasingly broad environments and long horizons. A Computer user specifies an outcome, and the system autonomously searches, browses, codes, creates documents, accesses external services, delegates work to subordinate agents, and persists in these efforts until the outcome is fulfilled via real-world actions or deliverables.
\end{itemize}

\begin{figure}[!htbp]
    \centering
    \includegraphics[width=0.75\textwidth]{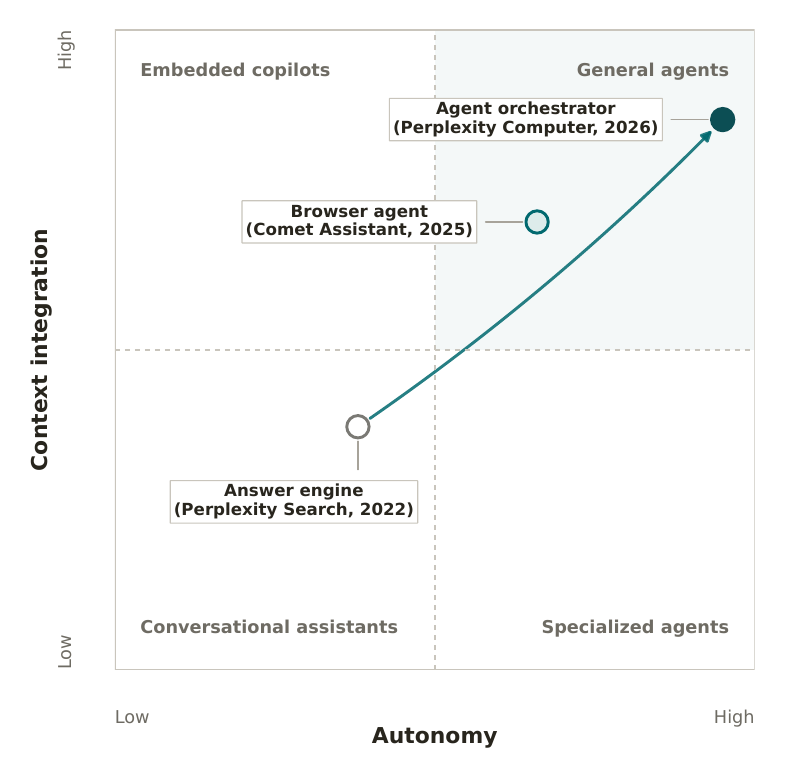}
    \caption{AI product progression by autonomy and workflow-context integration. Perplexity's Search represents the baseline for information retrieval and synthesis; Comet Assistant introduces deeper context integration and execution on top of an interactive browser interface; Computer combines long-horizon asynchronous execution with even deeper and broader context integration as an agent orchestrator.}
    \label{fig:product_progression}
\end{figure}

\clearpage
Our paper provides the first field evidence on task-level economic implications of the shift from conversational assistants to agent orchestration across a wide spectrum of knowledge work. We begin by defining a simple individual-level task-based framework to highlight the key economic forces and ground our empirical analysis. Each task is indexed by its required step count, with each step representing an atomic unit of work and longer tasks generating weakly greater value. The model centers on a shift in the cost structure: relative to conversational assistants, agents reduce marginal costs per step by replacing manual execution with autonomous implementation, but impose higher fixed costs through delegation and verification. The framework predicts that agent access expands the affordable task frontier toward weakly higher-value tasks and weakly increases total realized value; when the pre-agent budget binds, gains in total surplus and the aggregate value-to-cost ratio follow as corollaries. We then connect the framework to our empirical setting by mapping Search and Computer to the conversational assistant and agent categories respectively. We organize the empirical analysis around four themes:

\begin{enumerate}
    \item \textbf{Adoption} (Section~\ref{adoption}). Computer grew rapidly: cumulative queries reached 84$\times$ their first-week total over the 3-month study window (February~27 through May~27, 2026). A random sample of 100{,}000 classified queries further characterizes the use-case distribution: Research \& Analysis (25.8\%) and Document \& Asset Creation (18.6\%) are the largest categories, with both subject-matter domains and output formats (e.g., documents, websites, codebases, spreadsheets) spanning the breadth of knowledge work.

    \item \textbf{Autonomy} (Section~\ref{autonomy}). Because the same users interact with both products over the same period, we leverage matched sessions as natural experiments to control for user and task heterogeneity. In 10{,}000 session pairs with near-identical initial queries (cosine similarity $> 0.99$), Computer performs 26~minutes of autonomous planning and execution per session versus 33~seconds for Search, a 48$\times$ ratio in machine work. Classification of follow-up queries from 1{,}000 matched multi-turn session pairs reveals that Computer replaces manual directives with task verification and extension. Higher autonomy is achieved without sacrificing quality: on the next-turn dissatisfaction signal, Computer elicits medium-to-high dissatisfaction on 1.3\% of queries versus 2.9\% for Search, a 55\% reduction.

    \item \textbf{Efficiency} (Section~\ref{efficiency}). On the same matched sessions, a human equipped with Search alone takes an average of 269~minutes to complete a task. Replacing manual execution with automated implementation, the Computer~+~Human workflow reduces average task completion time to 36~minutes, lowering time and cost by 87\% and 94\%, respectively. A breakeven analysis shows that a Search-aided human professional would need to complete all manual steps in under 20~minutes to match the cost of Computer~+~Human. A sensitivity analysis further confirms robustness to variation in human-time estimates. These results are cross-validated through an independent LLM-driven procedure and user interviews.

    \item \textbf{Scope} (Section~\ref{scope}). Autonomous execution also expands the range of work users attempt.

    \emph{Horizontally}: Based on a sample of 8{,}000 users across 8 occupation clusters and all their queries, Computer queries venture outside users' primary occupation more often than Search queries from the same users. This pattern holds across all 8 occupation clusters, with an average gap of 9~percentage points (pp).

    \emph{Vertically}: Task difficulty also differs. A classification of 10{,}000 Computer and Search queries from a sample of 5{,}000 dual-product users suggests that:
    \begin{enumerate}
        \item Computer queries are more cognitively complex: 71\% abstract non-routine tasks versus 53\% for Search; 76\% higher-order Bloom cognition versus 55\%; Create-level work accounts for 50\% of Computer queries versus 26\% of Search.
        \item Computer queries draw upon a broader set of competencies: each Computer query requires substantive expertise in an average of 2.40 distinct O*NET Knowledge domains versus 1.74 for Search (+38\%), with Computer nearly three times as likely as Search to require three or more domains (51\% vs.\ 17\%).
        \item Computer composes more tasks into a single query: at the work-activity level, Computer queries engage an average of 2.95 O*NET Generalized Work Activities versus 2.24 for Search (+32\%) and 4.01 Intermediate Work Activities versus 2.87 (+40\%); the gap widens at finer grains, with 59\% more Detailed Work Activities (3.64 vs.\ 2.29) and 60\% more occupation-specific Task Statements (3.81 vs.\ 2.38) engaged.
        \item Computer unlocks new task possibilities for users: among Computer queries that engage at least one Task Statement, 23\% involve a Task Statement that never appears in the same users' Search query sample, and these Computer-only queries also tend to be more complex. The shares are smaller at coarser grains (5\% for Detailed Work Activities, under 1\% for Intermediate and Generalized Work Activities), indicating that Computer's distinctiveness lies in fine-grained executional work rather than coarse topical range. These shares also increase as the tolerance threshold is relaxed.

    \end{enumerate}

\end{enumerate}

Together, these findings suggest that autonomous task execution accelerates existing workflows, improves quality, reduces costs, and expands the range of work users undertake. By automating the generative components of tasks that require specialized expertise, agents make it easier for users to branch into domains outside their core competency and take on tasks that are costly to produce but relatively easy to verify. As individual workers absorb tasks that previously spanned occupational boundaries and expertise levels, the findings also suggest a reduction in coordination costs, with broader implications for occupational and organizational structure.

The paper proceeds as follows. Section~\ref{related-work} situates our contribution relative to prior work on AI's productivity impact, autonomous agent capabilities, and task recomposition. Section~\ref{framework} develops an individual-level task-based conceptual framework to derive welfare predictions and motivate the empirical analysis. Section~\ref{data} describes the samples drawn from Perplexity's Search and Computer products over the February--May 2026 study period. Sections~\ref{adoption}, \ref{autonomy}, \ref{efficiency}, and \ref{scope} present the four empirical themes in turn: adoption growth and use cases; autonomy gains on matched sessions; reductions in task time and cost; and expansion in the scope of work. Section~\ref{discussion} discusses limitations and implications. Proofs, a numerical example illustrating the propositions, and supplementary analysis and user-interview materials are collected in the Appendix.

\section{Related Work}\label{related-work}

\paragraph{Productivity impact of AI assistants.} A growing body of experimental evidence documents the productivity effects of generative AI assistants. For instance, \citet{noy2023experimental} find that ChatGPT reduces writing time by 40\% and raises output quality by 18\% in a randomized experiment with 453 professionals, with the largest gains for lower-ability workers. \citet{brynjolfsson2025generative} study 5{,}172 customer support agents and report a 15\% increase in issues resolved per hour, again with disproportionate gains for novice workers. \citet{dellacqua2023navigating} find that BCG consultants using GPT-4 improve performance by up to 34\% on tasks within the model's capability frontier, but perform worse on tasks beyond it, a ``jagged frontier'' that underscores the importance of task--tool fit. \citet{cui2024effects} run three field experiments with 4{,}867 software developers and find that GitHub Copilot increases completed tasks by 26\%, with junior developers benefiting most. A notable counterpoint comes from \citet{becker2025measuring}, whose randomized trial finds that experienced open-source developers using AI tools were 19\% slower, suggesting that productivity gains may depend on task familiarity and developer expertise. Relatedly, \citet{vendraminelli2025genai} document a ``GenAI wall effect'' in which AI assistance fails to fully close the performance gap between occupational insiders and outsiders, highlighting limits to horizontal expertise transfer through interactive tools. Moving beyond controlled experiments, \citet{tamkinmccrory2025productivity} estimate productivity gains directly from large-scale Claude usage logs, finding 80\% time savings across a broader set of tasks. These studies focus on humans working interactively with an AI assistant that augments each step. Our setting differs because Computer replaces the interactive loop with asynchronous delegation.

\paragraph{From assistance to autonomous AI agents.} The progression from interactive assistants to autonomous agents has been driven by advances in tool use and multi-step reasoning. \citet{schick2023toolformer} demonstrate that language models can learn to invoke external tools in a self-supervised manner; \citet{yao2023react} show that interleaving reasoning traces with action steps improves task completion on knowledge-intensive and decision-making tasks. \citet{kwa2025measuring} introduce the ``time horizon'' metric (the task duration at which agents achieve 50\% success rate), estimate the frontier at roughly one human-hour as of early 2025, and document a doubling time of approximately seven months; METR's updated estimates place the frontier at roughly 12 human-hours by early 2026 \citep{metr2026horizons}. Comparing how agents and humans execute the same tasks across five occupational domains, \citet{wang2025how} find that agents favor programmatic workflows over the UI-centric approaches humans use, delivering results 88\% faster and more than 90\% cheaper, though with quality that still lags human work. Beyond capability measurements, production deployments are beginning to reveal agent usage and impact in real-world settings. For instance, \citet{sarkar2026agents} analyzes nearly 120{,}000 Cursor users around the rollout of an agentic coding mode and finds that companies merge 39\% more code changes under the agent default, with experienced developers shifting effort from typing code to planning and supervising the agent. \citet{mccain2026measuring} analyze millions of human-agent interactions in Anthropic's Claude Code and find that autonomous turn durations at the 99.9th percentile nearly doubled from under 25 to over 45 minutes between October 2025 and January 2026, with experienced users granting agents full autonomy in over 40\% of sessions. \citet{demirer2026writing} combine usage data from over 100{,}000 GitHub developers across successive generations of coding tools and find that autocomplete, interactive coding agents, and autonomous coding agents raise coding activity by 40\%, 140\%, and 180\% respectively; these task-level gains, however, attenuate sharply down the production chain---to 50\% for projects and 30\% for releases---consistent with human bottlenecks limiting how much reaches shipped software. \citet{yang2025adoption} study the adoption and use of Perplexity's Comet agentic browser and find that early users concentrate on productivity- and learning-related use cases. Our study complements this work by comparing user interaction with an autonomous agent versus a conversational assistant, and by extending downstream impact analysis beyond coding to a wide range of knowledge work.

\paragraph{Occupational exposure and task recomposition.} Several studies move beyond individual- and task-level analysis to estimate, at the macro level, which occupations and tasks are most exposed to AI automation. For instance, \citet{eloundou2024gpts} find that approximately 80\% of the U.S.\ workforce could have at least 10\% of their tasks affected by LLMs, with higher-wage occupations more exposed. \citet{felten2023occupational} construct an AI occupational exposure index updated for language-model capabilities and find similar patterns of broad exposure concentrated in white-collar work. These exposure analyses assess potential displacement but do not measure how task composition actually changes once AI tools are adopted. Usage-based measurements are beginning to fill this gap: for instance, \citet{appel2026economic} analyze two million Claude conversations and API transcripts mapped to occupational tasks and find that, in consumer usage, augmentation slightly outpaces automation. \citet{massenkoff2026labor} introduce a usage-based measure of AI exposure and document that observed deployment lags far behind theoretical capability across occupations. A complementary theoretical literature emphasizes that automation can both displace existing tasks and create new ones, with the net labor-market impact depending on the balance between displacement and reinstatement forces \citep{acemoglu2019automation}. Recent work further shows why task-level exposure measures can miss important complementarities in how work is organized. For instance, \citet{gans2026ring} extend the O-ring model to production processes with complementary tasks and show that the return to automation is limited by the human bottlenecks in the process. \citet{garicano2026weak} model jobs as bundles of tasks and show that weak bundles of easily separable tasks face stronger displacement pressure, whereas strong bundles are more resilient. Our analysis provides direct evidence on task recomposition as users shift from conversational assistants to agents: usage tends to expand horizontally into other occupations and vertically into more complex work.

\section{Conceptual Framework}\label{framework}

This section develops a simple task-based framework at the individual-worker level. Its aim is to derive partial-equilibrium predictions about how autonomous execution affects task completion under minimal assumptions and to ground the empirical analysis. To keep the theory product-agnostic, we refer to two modes: a conversational mode and an agent mode. We keep the framework general and provide a concrete numerical example in Appendix~\ref{app:numerical}.

\subsection{Task primitives}

Consider a finite set of candidate task opportunities indexed by:
\[
j=1,\ldots,J.
\]
Tasks are ordered by the number of steps or subtasks required for completion:
\[
0 < s_1 \leq s_2 \leq \cdots \leq s_J,
\]
which we call the task's \emph{step count}. A step is an atomic unit of work (e.g., lookup, calculation, code execution, synthesis). We make four assumptions on the value and cost structure.

\paragraph{Assumption 1: Higher-step tasks have weakly higher value.}\asmplabel{asmp:value} Task values are weakly increasing in step count:
\[
0 < v_1 \leq v_2 \leq \cdots \leq v_J.
\]
This implies that longer tasks create weakly greater value. We adopt the convention $v_0 \equiv 0$: value is zero if no task is completed.

\paragraph{Assumption 2: Task value is realized upon full completion.}\asmplabel{asmp:completion} Task value $v_j$ is realized only upon completion of all $s_j$ steps; partially completed tasks generate no value. Equivalently, task opportunities are indivisible: a user attempts task $j$ in full or not at all.

\paragraph{Assumption 3: Agent has higher fixed per-task cost than Conversational.}\asmplabel{asmp:fixed} Each mode has a fixed per-task cost:
\[
 0 < f_{\Conversational} < f_{\Agent}.
\]
$f_{\Conversational}$ captures the cost of formulating the conversational prompt, typically a single question. $f_{\Agent}$ captures the cost of delegating a task to an agent: specifying a well-scoped objective and later reviewing the delegated output. The sign $f_{\Agent} > f_{\Conversational}$ reflects a simple observation about agent user experience: directing an autonomous system typically requires more specification and verification than issuing a simple one-shot query.

\paragraph{Assumption 4: Agent has lower marginal per-step cost than Conversational.}\asmplabel{asmp:marginal} Each step incurs a mode-specific marginal cost $m_{t} > 0$, and the agent mode is cheaper per step than the conversational mode:
\[
0 < m_{\Agent} < m_{\Conversational}.
\]
Under the conversational mode the user must plan each step, issue it, read the response, and execute; the per-step cost includes human planning, interpretation, and execution overhead. Under the agent mode the system plans and executes autonomously, leaving the user only the one-time cost of delegation and review. Equivalently, $m_{\Agent}$ is the \emph{delegated-execution} cost per step, $m_{\Conversational}$ is the \emph{human-in-the-loop} cost per step, and their gap is the autonomy premium.

\paragraph{Total cost.} Combining the primitives, the total cost to complete task $j$ under mode $t$ is
\[
C(s_j; t) = f_{t} + m_{t} s_j.
\]
For nonempty toolkit $\mathcal{T} \subseteq \{\Conversational, \Agent\}$ define the effective completion cost
\[
c_j^{\mathcal{T}} \equiv \min_{t \in \mathcal{T}} C(s_j; t).
\]

\subsection{Cost properties}

The following property of the cost function follows directly from Assumptions~\ref{asmp:fixed}--\ref{asmp:marginal}. Proofs are collected in Appendix~\ref{app:proofs}.

\begin{lem}[Agent is preferred for more complex tasks]\label{lem:sorting}
Under Assumptions~\ref{asmp:fixed}--\ref{asmp:marginal}, for any task that could be routed to either mode, the user strictly prefers the agent mode to the conversational mode whenever the step count exceeds a positive threshold,
\[
s > s^{\ast} \equiv \frac{f_{\Agent} - f_{\Conversational}}{m_{\Conversational} - m_{\Agent}} > 0,
\]
and strictly prefers the conversational mode below it.
\end{lem}

Lemma \ref{lem:sorting} establishes that the agent's dominance is confined to tasks with enough steps to amortize its higher delegation overhead. Even though the agent mode is strictly cheaper per step, its higher fixed delegation cost dominates on short tasks. The conversational mode is therefore preferred in the low-$s$ region (e.g., quick lookups, single-turn clarifications, one-off factual questions), while the agent mode is preferred in the high-$s$ region.

\subsection{Optimal task selection}
Let $B>0$ denote the user's resource endowment. Let $a_j \in \{0,1\}$ indicate whether task $j$ is attempted. If task $j$ is attempted, the user also selects a mode from the toolkit that minimizes the cost of task $j$. By Assumption~\ref{asmp:completion}, the user then pays the total mode-dependent cost $c_j^{\mathcal{T}}$ and generates value $v_j$; otherwise nothing is paid and no value is realized. The user therefore solves a standard 0/1 knapsack problem
\[
\max_{\{a_j\}_{j=1}^{J}} \sum_{j=1}^{J} a_j v_j
\]
subject to the aggregate resource budget
\[
\sum_{j=1}^{J} a_j c_j^{\mathcal{T}}\leq B.
\]
We describe a standard dynamic programming solution in Appendix~\ref{app:dp}.

\subsection{Predictions}

Let $a^{\ast}_{\mathcal{T},j}$ denote the optimal attempt decision for task $j$ under toolkit $\mathcal{T}$, and write
\[
W^{\mathcal{T}} \equiv \sum_{j=1}^{J} a^{\ast}_{\mathcal{T},j} v_j,
\qquad
K^{\mathcal{T}} \equiv \sum_{j=1}^{J} a^{\ast}_{\mathcal{T},j} c_j^{\mathcal{T}},
\]
\[
\Pi^{\mathcal{T}} \equiv W^{\mathcal{T}} - K^{\mathcal{T}},
\qquad
A^{\mathcal{T}}\equiv\{j:a^{\ast}_{\mathcal{T},j}=1\}
\]
for the total realized value, total realized cost, total surplus (value net of cost), and selected task set in the optimum under $\mathcal{T}$. We use \emph{pre} and \emph{post} for the conversational-only and conversational-plus-agent toolkits.
For shorthand, write effective costs in both periods as
\[
c_j^{\text{pre}} \equiv C(s_j;\Conversational), \qquad
c_j^{\text{post}} \equiv \min\{C(s_j;\Conversational),C(s_j;\Agent)\}.
\]
Let the induced upper affordability endpoints be
\[
u^{\text{pre}}\equiv\max\{j:c_j^{\text{pre}}\leq B\},
\qquad
u^{\text{post}}\equiv\max\{j:c_j^{\text{post}}\leq B\}
\]
with the convention $u=0$ when no task is individually affordable.

We derive several predictions from the knapsack problem, with proofs in Appendix~\ref{app:proofs}.

\begin{prop}[Affordable value frontier expands]\label{prop:afford}
Adding agent access weakly expands the set of individually affordable tasks; therefore, the highest individually affordable value weakly rises:
\[
u^{\text{post}}\geq u^{\text{pre}}, \quad v_{u^{\text{post}}}\geq v_{u^{\text{pre}}}.
\]
This indicates that the agent mode unlocks weakly higher-value tasks that are not feasible under the conversational mode.
\end{prop}

\begin{prop}[Total value expands]\label{prop:value}
Adding agent access weakly increases total realized value: $W^{\text{post}} \geq W^{\text{pre}}$.
\end{prop}

\begin{cor}[Total surplus expands]\label{cor:surplus}
When the conversational-only chosen bundle exhausts the aggregate budget, adding agent access weakly increases total surplus:
\[
\Pi^{\text{post}} \geq \Pi^{\text{pre}}.
\]
\end{cor}

\begin{cor}[Value-to-cost ratio expands]\label{cor:ratio}
When the conversational-only chosen bundle exhausts the aggregate budget, adding agent access weakly increases the aggregate value-to-cost ratio:
\[
\frac{W^{\text{post}}}{K^{\text{post}}} \;\geq\; \frac{W^{\text{pre}}}{K^{\text{pre}}}.
\]
\end{cor}

\begin{prop}[Surplus change decomposes into intensive, entry, and exit margins]\label{prop:decomp}
For any pre- and post-agent selected task sets, the surplus change can be written as
\begin{align*}
\Delta\Pi
\equiv
\Pi^{\text{post}}-\Pi^{\text{pre}}
&=
\underbrace{\sum_{j\in A^{\text{pre}}\cap A^{\text{post}}}
\big(c_j^{\text{pre}}-c_j^{\text{post}}\big)}_{\text{intensive: cost saving}} \\
&\quad+
\underbrace{\sum_{j\in A^{\text{post}}\setminus A^{\text{pre}}}
\big(v_j-c_j^{\text{post}}\big)}_{\text{extensive: surplus from entry}} \\
&\quad-
\underbrace{\sum_{j\in A^{\text{pre}}\setminus A^{\text{post}}}
\big(v_j-c_j^{\text{pre}}\big)}_{\text{extensive: surplus from exit}}.
\end{align*}
The intensive term is cost savings on retained tasks and is weakly non-negative. The entry and exit terms are the net surplus from newly attempted and no-longer-attempted tasks.
\end{prop}

Task value is unobserved in our data, so the value, surplus, ratio, and decomposition results are not directly testable. The empirical analysis therefore focuses on observable cost and scope outcomes. Section~\ref{efficiency} tests cost structure assumptions and cost reduction. Section~\ref{scope} tests whether lower costs unlock more complex attempted work.

\section{Data}\label{data}

Perplexity Computer was released on February~25, 2026. Our main analysis covers the 3-month post-launch period February~27 through May~27, 2026, and draws on 8 samples. We follow the privacy procedure used in \cite{yang2025adoption} to ensure that no raw queries are exposed to human analysts and all results are reported in highly aggregated forms. In our empirical setting, Search instantiates the conversational mode and Computer instantiates the agent mode in the conceptual framework. We describe each of the 8 samples analyzed in the empirical sections below.

\paragraph{Adoption (Section~\ref{adoption}).} We use three samples for adoption analyses. The first is the full universe of Computer and Search queries over the 3-month post-launch window. For Search query growth, we split the user base into users who issued at least one Computer query in the window (``Computer users'') and users who did not (``non-Computer users''). The second sample is a random draw of 100{,}000 Computer queries from the same window, each labeled by an LLM along two dimensions: primary task category and subject-matter domain. The third sample supports the cross-product difference-in-differences analysis (Appendix~\ref{app:complementarity}): from 100{,}000 sampled Computer adopters we retain the 61{,}913 with at least one pre-adoption Search query and exact-match them to an approximately equal-sized control group of non-adopters on subscription tier, primary search topic, and pre-period Search intensity quartiles, yielding a balanced user-day panel of 123{,}699 users spanning February~13 through May~27, 2026.

\paragraph{Autonomy (Section~\ref{autonomy}).} The autonomy analysis rests on a matched-pair design that compares Computer and Search outcomes on sessions with near-identical initial user queries.\footnote{A session (or thread) consists of one or more related turns that share a common context. Within a session, the AI retains prior messages, tool outputs, and intermediate state, enabling coherent reasoning across turns rather than treating each message in isolation.
} We begin by identifying dual-product users (those who issued at least one Computer initial query, defined as the first message of a session, and at least one Search initial query during the window) and draw 100{,}000 such users. To ensure that every matched Computer query exercises Computer's full agentic capabilities, we require the Computer session to invoke at least one execution (``do'') tool such as code execution, browser actions, file creation, and external API calls that go beyond information retrieval and synthesis. Queries that do not invoke a ``do'' tool are functionally similar to Search and are excluded from the Computer side. For each sampled user we collect every qualifying Computer initial query and up to the 100 most recent Search initial queries, embed each query into a dense vector representation, and compute pairwise cosine similarity within the user. We then perform one-to-one greedy matching, retaining each Computer query's top-ranked Search neighbor, and keep pairs with similarity $> 0.99$.\footnote{We set the threshold to make the matched pairs as similar as possible while allowing for minor punctuation and formatting differences, such as whitespace and line breaks.} From the resulting pool of near-identical matches we randomly sample $10{,}000$ pairs (multiple pairs per user are possible). For sessions with multiple queries, we infer per-query user dissatisfaction from the content of the next turn. For the follow-up query analysis, we additionally draw a $1{,}000$-pair subsample restricted to pairs where both sessions have $\geq 2$~turns, and classify these into a 10-category taxonomy.

\paragraph{Efficiency (Section~\ref{efficiency}).} The efficiency analysis reuses the $10{,}000$ matched query pairs from the autonomy sample, augmented with execution-cost and human-time data. For the Search~+~Human counterfactual we use two independent human-time estimates. The tool-based estimate sums human-equivalent minutes for each ``do'' tool invocation observed in the Computer thread; ``search'' tools are treated as already provided by Search and contribute zero human minutes. The LLM-based estimate is an independent validation: for each of the $10{,}000$ pairs we prompt an LLM with the query text only, describing the Search~+~Human counterfactual, and elicit a total human-time estimate per task session. Both estimates are converted to cost by combining model cost with human labor cost at domain-specific hourly wages (BLS Occupational Employment and Wage Statistics, May 2025 release, mapped to the 18 query domains). A third, user-reported estimate comes from 45-minute semi-structured interviews with 25 active Computer users (6 enterprise, 19 consumer) recruited from those with $\geq 5$ historical queries; participants self-report their pre-Computer workflow time, cost, and before/after comparison.

\paragraph{Scope (Section~\ref{scope}).} Scope analyses use two paired samples drawn from the same dual-product population. The \emph{horizontal} sample identifies each user's primary occupation cluster by mapping every Search query in the window to one of the 8 most common clusters in our data (Digital Technology, Financial Services, Healthcare \& Human Services, Education, Public Service \& Safety, Management \& Entrepreneurship, Marketing \& Sales, Arts \& Design), then assigning the mode cluster as the user's primary occupation. We restrict to users active in both products and then sample exactly $1{,}000$ users per cluster, yielding $8{,}000$ users in total. For each sampled user we then pull all Search and Computer queries over the full 3-month window; each Computer query is assigned a destination cluster by a direct single-label LLM call (8 clusters + ``Other''), while each Search session is mapped deterministically via its topic domain. The \emph{vertical} sample is a paired query-level draw from a fixed set of $5{,}000$ dual-product users: for each user we randomly sample one ``do''-gated Computer initial query and one Search initial query within the window, yielding $10{,}000$ queries from the same users, a paired within-user comparison. Each query is then LLM-classified along four axes: cognitive complexity (Bloom's Revised Taxonomy and routine versus abstract task types), required knowledge breadth (which O*NET Knowledge domains the query requires), work-activity composability (which O*NET Generalized Work Activities, Intermediate Work Activities, Detailed Work Activities, and occupation-specific Task Statements the query engages), and new tasks unlocked (work activities that appear in the user's Computer queries but are essentially absent from the same user's Search queries).

\section{Adoption}\label{adoption}

\paragraph{Growth.} Figure~\ref{fig:adoption} shows Computer's growth trajectory over the 3-month post-launch period, using Search growth as a baseline for an established product. Each series is reported as a cumulative running total and indexed to its own week-1 cumulative (Feb~27--Mar~5 $= 1\times$). Cumulative Computer queries reached 84$\times$ by May~27. Over the same window, cumulative Search queries from Computer users reached 14$\times$, slightly above non-Computer users at 12$\times$.

The faster growth of Search among Computer users could in principle reflect either complementarity or selection; we estimate the causal effect with a matched difference-in-differences design that compares Computer adopters to non-adopters exactly matched on subscription tier, primary search topic, and pre-period Search intensity quartiles. Computer adoption increases daily Search queries by $1.05$, with consistent results across alternative staggered-adoption estimators. Appendix~\ref{app:complementarity} reports the full design, robustness checks, and intensive-margin estimates.

\begin{figure}[!htbp]
      \centering
      \includegraphics[width=\textwidth]{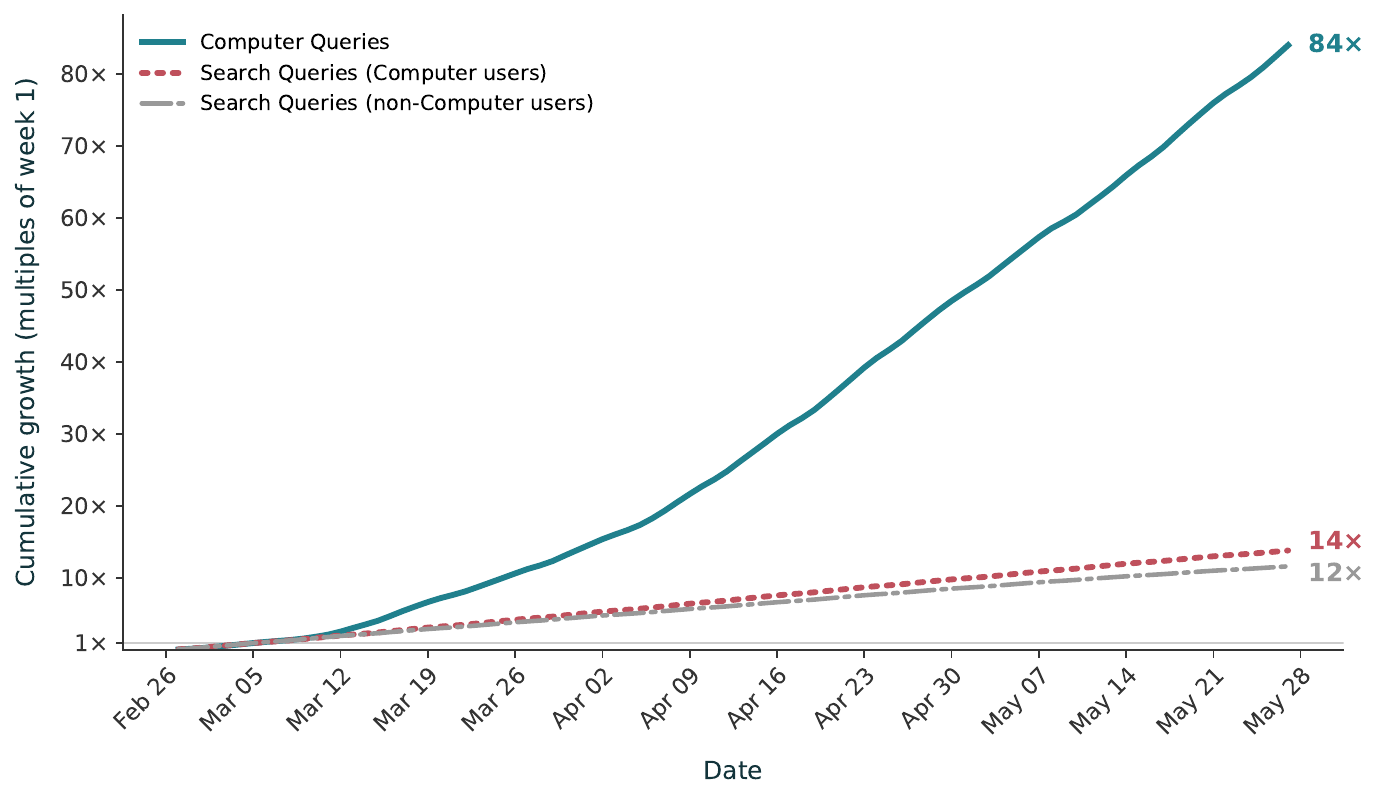}
      \caption{Cumulative adoption growth, with each series indexed to its own week-1
  cumulative total (Feb~27--Mar~5 $= 1\times$). Cumulative
  Computer queries reached 84$\times$. Cumulative Search queries from Computer users
  reached 14$\times$, slightly above non-Computer users at 12$\times$.}
      \label{fig:adoption}
  \end{figure}

\paragraph{What do people use Computer for?} To characterize the task distribution, we classify a random sample of 100{,}000 Computer queries along two dimensions: primary task category and subject-matter domain (Figure~\ref{fig:use_cases}). Research \& Analysis is the most common task category (25.8\%), followed by Document \& Asset Creation (18.6\%). Capabilities \& Product Discovery accounts for 5.3\%--6.0\% of queries and is decreasing over time as users transition from exploration to production.\footnote{The label appears on both classification axes: 5.3\% of queries have it as their \emph{task category}, and 6.0\% have it as their \emph{subject-matter domain}. The two sets overlap but are not identical, since task and domain are classified independently.} Subject-matter domains are broadly distributed across knowledge work: Software \& IT leads (13.8\%), followed by Finance \& Investing (10.8\%), Marketing \& Sales (7.6\%), General Business Operations (7.0\%), Healthcare \& Life Sciences (6.8\%), Education \& Academia (5.9\%), Legal \& Compliance (5.5\%), and Media \& Creative (5.1\%).

\begin{figure}[!htbp]
    \centering
    \includegraphics[width=\textwidth]{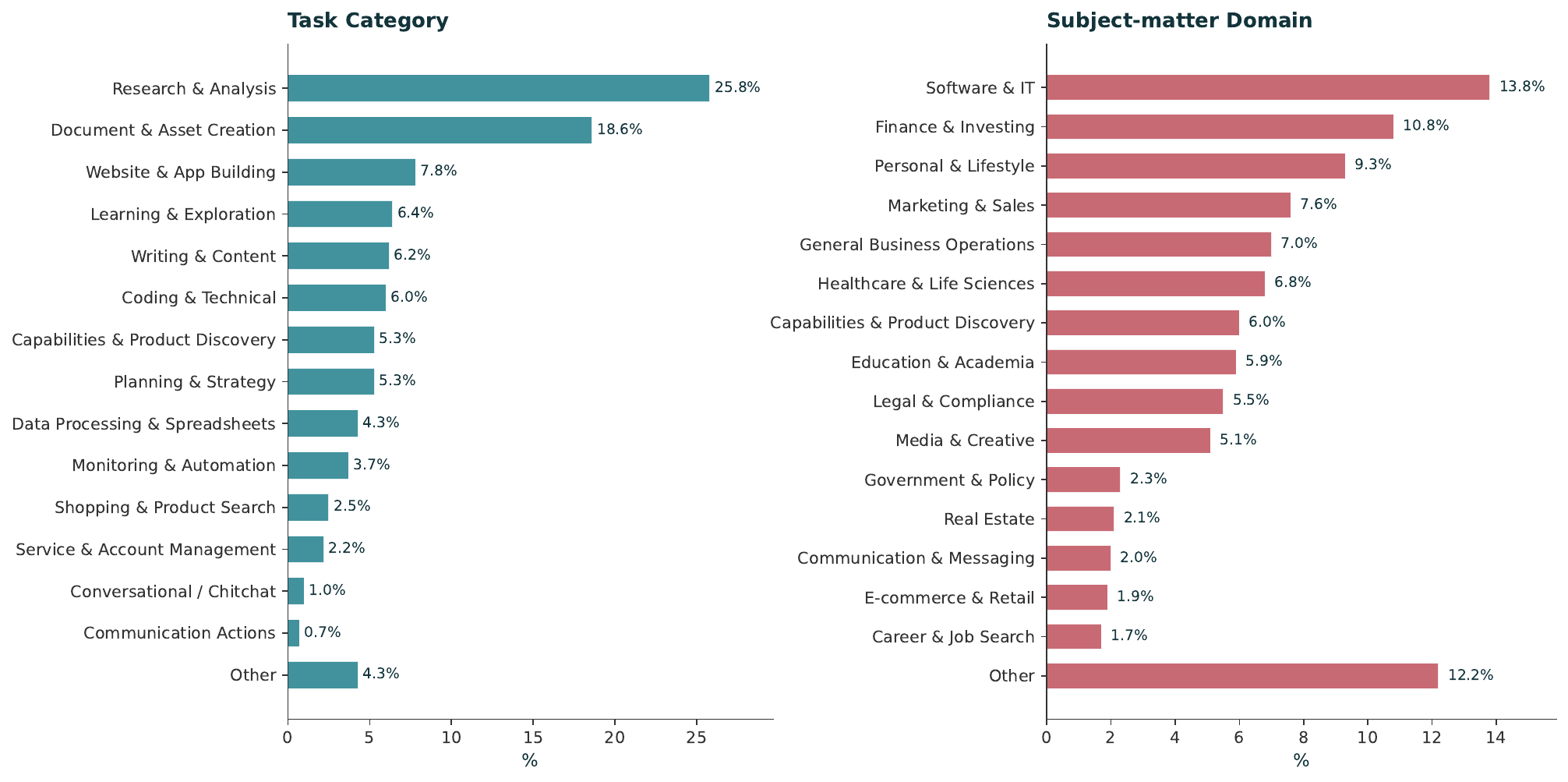}
    \caption{Use case distribution of Computer queries along two classification dimensions. \textbf{Left:}~Task category: Research \& Analysis (25.8\%) and Document \& Asset Creation (18.6\%) dominate. \textbf{Right:}~Subject-matter domain: usage is broadly distributed across 15+ subject areas. Within each panel, categories are sorted by share (descending), with Other placed last. Percentages may not sum exactly to 100\% due to rounding.}
    \label{fig:use_cases}
\end{figure}

\section{Autonomy}\label{autonomy}

A core design principle of Computer is to complete tasks autonomously, executing multi-step workflows with minimal human intervention. To measure how well it achieves this, we compare Computer sessions to Search sessions on the same initial queries from the same users. While both query and session (a thread of related queries) can serve as proxies for a task unit, Search and Computer may induce different workflows as captured by the follow-up queries. We therefore use session as the primary unit of analysis, and also report query-level results where relevant. 

\subsection{Method}

A direct comparison between Computer and Search is complicated by endogenous task selection: users may send different types of queries to each product, as predicted by the sorting behavior in Lemma~\ref{lem:sorting}. Ideally, we would compare outcomes for the same task performed with Computer versus Search. We therefore exploit user-generated natural experiments in which the same user submits near-identical initial queries (first message in a session) to both products.  

We identify users who submitted at least one initial query to both Computer and Search between February~27 and May~27, 2026. From these, we sample 100{,}000 users. We restrict the Computer side to initial queries that invoked at least one execution (``do'') tool (e.g., browser actions, code execution, file writes, drive uploads, external connector calls), so that every matched Computer session actually performs autonomous work rather than returning a chat-style response. For each sampled user, we collect all qualifying initial Computer queries and up to 100 most recent initial Search queries. We embed each query and compute pairwise cosine similarity within each user's query sets. We then perform one-to-one greedy matching and sample 10{,}000 near-identical pairs (cosine similarity $> 0.99$), which provides a clean control for task content by comparing effectively the same task attempted through both products. The task domain of a session is labeled by the primary domain of the initial query.

\subsection{Results}

We first compare autonomy by measuring execution time, the rate of model pauses and user stops, and the number of connector calls to external apps.

\paragraph{Execution time.} The defining feature of Computer's autonomy is the machine work it performs between user turns. For Computer, we compute per-turn wall-clock time\footnote{Many Computer sessions involve parallel task execution, so wall-clock time reflects user experience and underestimates total machine time.} from the user's submission timestamp to the last LLM-response end within the turn, summed across turns, and capped at three hours per session to reduce the influence of outliers\footnote{Less than 5\% of Computer sessions exceed the three-hour limit due to a combination of rare long-running jobs, recurring jobs, system retries, and other edge cases.
}; this captures both model reasoning and downstream tool-execution time. For Search, we similarly use the end-to-end server latency from query receipt to last token summed over turns (covering retrieval, reasoning, and generation).

Averaging across matched pairs, Computer sessions run 26~minutes of wall-clock execution, versus 33~seconds for Search, a 48$\times$ ratio.\footnote{Ratios in this section are computed from unrounded values and may differ from ratios implied by the rounded numbers shown.} An average Computer session contains 5.3 queries versus 2.8 for Search, yielding a per-query runtime gap of 25$\times$. The gap is also visible at the distribution level (Figure~\ref{fig:autonomy_dist}): Computer and Search per-session runtimes barely overlap, with Search concentrated near 10--30~seconds and Computer spread across roughly 5~minutes to over an hour (median 9~minutes vs.\ 14~seconds, a 40$\times$ gap).

The execution-time ratio also varies substantially by domain (Figure~\ref{fig:autonomy}), reflecting the difference in the nature of tasks across domains. Local shows the largest gap (75$\times$), followed by Politics (67$\times$), Finance (64$\times$), and Business (60$\times$). Science (26$\times$) and Education (27$\times$) show the smallest ratios, because Search's per-turn responses already suffice for common tasks like concept explanation. Technology and Business, the two largest categories by volume, show 58$\times$ and 60$\times$ ratios, with Computer averaging 27 and 31~minutes versus 28 and 31~seconds for Search.

\begin{figure}[h]
    \centering
    \includegraphics[width=0.8\textwidth]{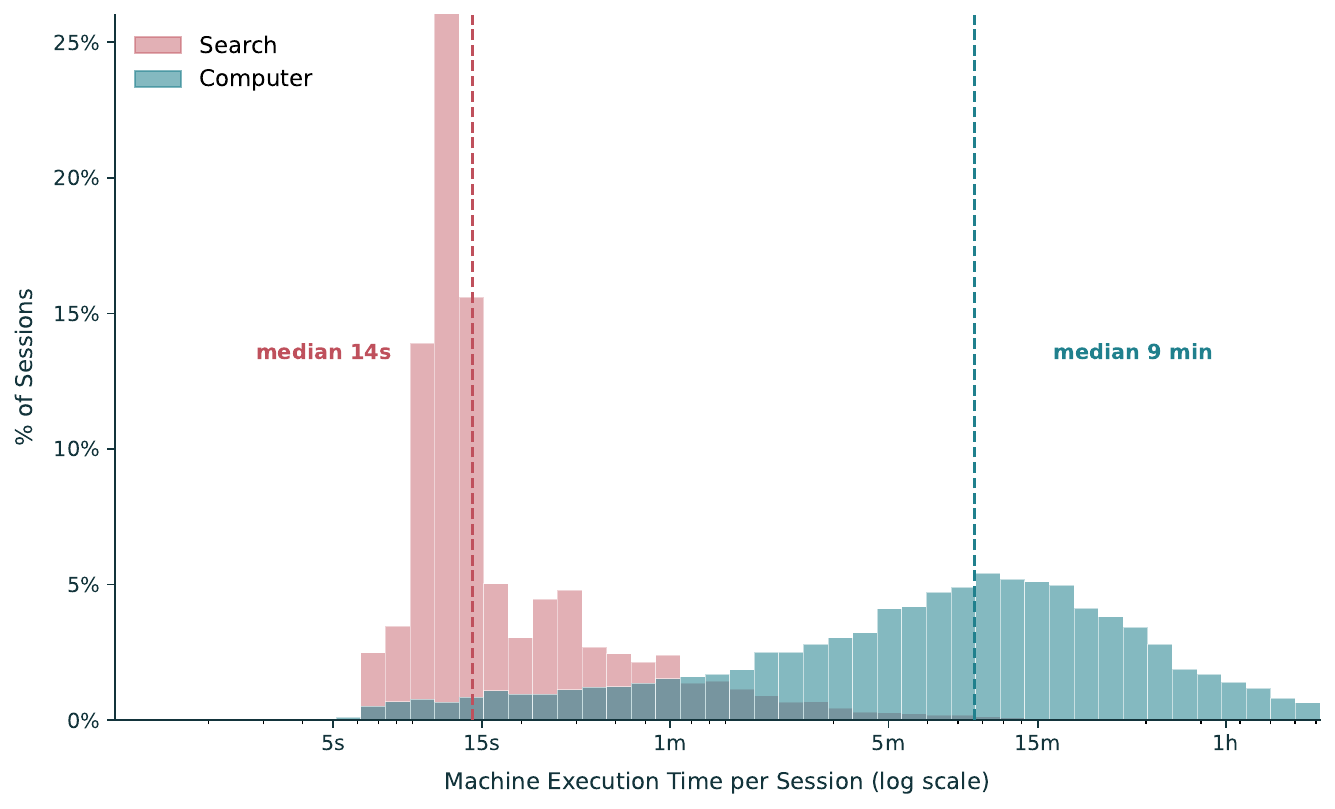}
    \caption{Distribution of per-session machine execution time for Computer vs.\ Search. Each session contributes one value: the total machine execution time summed across all turns. Computer is per-turn wall-clock from user submission to the last LLM-response end, summed across turns and capped at three hours; Search is end-to-end server latency from query receipt to last token, summed across turns. The two distributions barely overlap: Search is a tight mass near 10--30~seconds, while Computer is a wide, right-skewed distribution centered near 9~minutes, with a 40$\times$ gap in medians (9m vs.\ 14s; ratio computed from unrounded values).}
    \label{fig:autonomy_dist}
\end{figure}

\begin{figure}[!htbp]
    \centering
    \includegraphics[width=.8\textwidth]{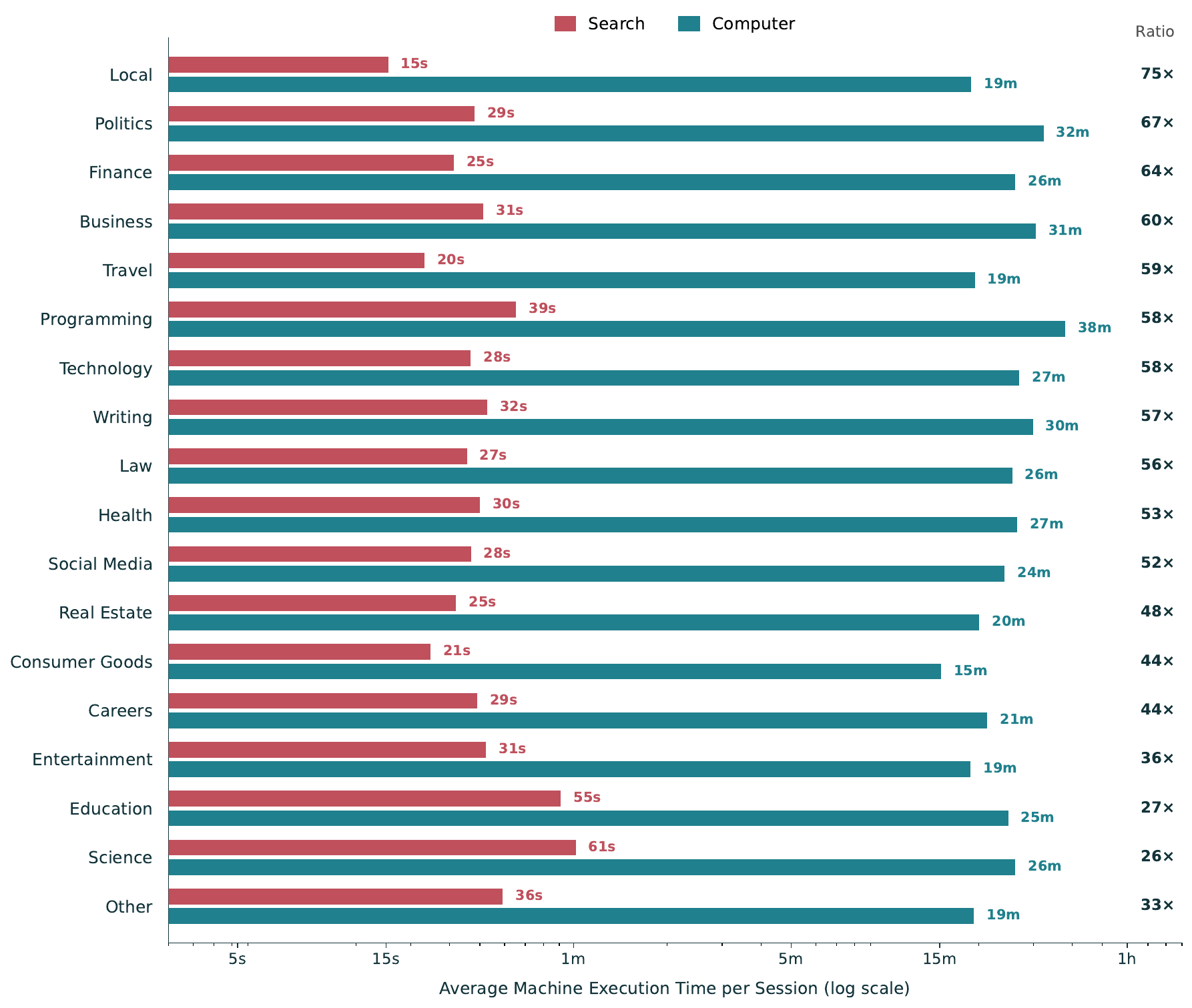}
    \caption{Level of autonomy: average machine execution time per session across 18~domains, for Computer vs.\ Search. Computer sessions are restricted to those that invoked at least one execution (``do'') tool. The ratio column reports Computer/Search execution time per session. Across all domains, Computer performs roughly 26--75$\times$ as much machine work per session as Search. Domains are sorted by the Computer/Search ratio (descending), with Other placed last. Ratios are computed from unrounded means and may differ from those implied by the rounded times shown.}
    \label{fig:autonomy}
\end{figure}

\paragraph{Model pauses and user stops.} At the query level, 13\% of Computer queries invoke at least one pause-for-user tool during execution, versus 0.3\% of Search queries; at the session level, the rates rise to 38\% versus 0.8\%. Because Computer performs longer-running actions, it more often pauses to request critical input from the user, ensuring that the final output matches the user's intent. At the session level, pauses are dominated by approval prompts (24.2\% of sessions), followed by open clarifying questions (16.9\%), structured input requests (2.3\%), and file-upload waits (0.7\%).\footnote{A session can contain multiple pause types, so these shares need not sum to the 38\% any-pause rate.} Despite Computer's longer runtimes, user-initiated interruptions are similar across products: 3.7\% of Computer sessions contain at least one user stop event versus 3.4\% of Search sessions. Users are therefore no more likely to abandon a long autonomous Computer run than a short Search response, suggesting that the autonomous execution is largely trusted to run to completion once launched.

\paragraph{Connector calls.} Another indicator of autonomy is whether Computer chains many external-tool calls via Model Context Protocol (MCP) or an Application Programming Interface (API) endpoint into a single session, work that a Search user would otherwise operate manually through separate apps. 7.9\% of Computer sessions invoke at least one connector call versus 1.8\% of Search sessions, a 4$\times$ gap, and the mean number of connector calls per session is 1.19 for Computer vs.\ 0.10 for Search, a 12$\times$ ratio. For sessions in which a connector is invoked, Computer fires 15.0 calls per session on average versus 5.5 for Search. At the pair level, of the 914 pairs in which at least one side used a connector, only Computer used one in 80\% (735 pairs). The pattern is most extreme for Finance (23\% Computer sessions vs.\ 1.2\% Search), Technology (10\% vs.\ 3\%), and Business (9\% vs.\ 2\%). This pattern suggests that Computer's autonomy also operates with greater context integration over a larger input-output surface, enabling it to retrieve data from and take actions across a broader set of external services.

\paragraph{What do follow-up turns contain?} To characterize how user interaction differs with greater autonomy, we examine the content of follow-up turns in the matched sessions. We sample 1{,}000 matched pairs where both products required multiple turns and classify all 15{,}507 follow-up queries (7{,}093 on the Search side, 8{,}414 on the Computer side) with an LLM against a single 10-category taxonomy (Table~\ref{tab:followup_taxonomy}). Three patterns emerge. First, the overall propensity toward task advancement is near-identical across products (drill-downs, new subtasks, and extensions together account for 52.7\% of Computer follow-ups versus 52.9\% for Search), but its composition shifts: relative to Search, Computer substitutes extensions for drill-downs, with drill-downs falling from 23.4\% to 22.0\% ($-1.4$~pp) and extensions rising from 12.5\% to 14.2\% ($+1.7$~pp).\footnote{The two differ in what they ask the system to do: a \emph{drill-down} is an information-seeking question about the existing output (``why did you choose this?'', ``how does this compare?''), which keeps the user in a clarification loop, whereas an \emph{extension} requests a new component that builds on the delivered output (``now also add X''), pushing the artifact further.} Because Computer returns a more complete deliverable up front, its users spend relatively fewer follow-ups clarifying the result and more extending it. Second, Computer users spend slightly more of their follow-ups reviewing and revising autonomous output: revision and verification together account for 24.6\% of Computer follow-ups versus 23.6\% for Search ($+1.0$~pp). Third, Search is heavier on lightweight continuation: confirmations, format-delivery, and retry requests together account for 11.6\% of Search follow-ups versus 9.9\% for Computer ($-1.7$~pp), short directives that Computer integrates into its initial run. Manual data inputs are essentially identical (8.5\% vs.\ 8.4\%), as expected when the matched-pair design controls for task content. Matched Computer and Search sessions therefore contain human turns for different reasons: Search turns reflect shorter digest-and-execute loops, whereas Computer turns reflect longer review-and-extend loops.

\begin{table}[!htbp]
\centering
\small
\begin{tabular}{lrrr}
\toprule
Category & Search (\%) & Computer (\%) & $\Delta$ (pp) \\
\midrule
\textit{Task advancement} & & & \\
\quad Drill-down       & 23.4 & 22.0 & $-1.4$ \\
\quad New subtask      & 17.0 & 16.5 & $-0.5$ \\
\quad Extension        & 12.5 & 14.2 & $+1.7$ \\
\textit{\quad Subtotal} & \textit{52.9} & \textit{52.7} & \textit{$-0.2$} \\
\midrule
\textit{Output review / iteration} & & & \\
\quad Revision         & 13.9 & 14.1 & $+0.2$ \\
\quad Verification     &  9.7 & 10.5 & $+0.8$ \\
\textit{\quad Subtotal} & \textit{23.6} & \textit{24.6} & \textit{$+1.0$} \\
\midrule
\textit{Manual input} & & & \\
\quad Data input       &  8.4 &  8.5 & $+0.1$ \\
\midrule
\textit{Short directives} & & & \\
\quad Confirmation     &  6.8 &  6.5 & $-0.3$ \\
\quad Format-delivery  &  3.4 &  2.7 & $-0.7$ \\
\quad Retry            &  1.4 &  0.7 & $-0.7$ \\
\textit{\quad Subtotal} & \textit{11.6} & \textit{9.9} & \textit{$-1.7$} \\
\midrule
\textit{Other} & & & \\
\quad Unclassified     &  3.5 &  4.3 & $+0.8$ \\
\midrule
Total queries & 7{,}093 & 8{,}414 & \\
\bottomrule
\end{tabular}
\caption{Follow-up query composition for Computer vs.\ Search in the 1{,}000-pair multi-turn sample classified against a single 10-category taxonomy. Group subtotals (italics) sum the constituent categories. Task advancement is near-identical overall (52.7\% vs.\ 52.9\%), but its composition shifts from drill-downs toward extensions for Computer (extensions 14.2\% vs.\ 12.5\%; drill-downs 22.0\% vs.\ 23.4\%); Computer is also slightly higher on output review/iteration overall (24.6\% vs.\ 23.6\%), while Search is heavier on short directives (11.6\% vs.\ 9.9\%). Manual data inputs are matched ($\approx$8.5\%), consistent with the matched-pair design controlling for task content. Columns may not sum exactly to 100\% due to rounding.}
\label{tab:followup_taxonomy}
\end{table}

\paragraph{User satisfaction.} Lastly, we investigate if higher autonomy is associated with lower user satisfaction. Higher autonomy could in principle come at a quality cost: users might issue more dissatisfied follow-ups to Computer's autonomous output than to Search's chat response. We test this directly by scoring each response's user dissatisfaction on the scale \{zero, low, mid, high\} based on what the user does next (re-asks, corrections, error reports, retries, etc.). The dissatisfaction signal is restricted to multi-turn sessions. Computer elicits less dissatisfaction at every level (Table~\ref{tab:user_satisfaction}): the mid+high rate is 1.3\% for Computer versus 2.9\% for Search, and any dissatisfaction (low+mid+high) is 10.8\% versus 16.6\%. Computer's autonomous execution thus increases autonomy and response quality simultaneously.

\begin{table}[!htbp]
\centering
\small
\begin{tabular}{lrrr}
\toprule
Next-turn dissatisfaction signal & Search (\%) & Computer (\%) & $\Delta$ (pp) \\
\midrule
zero (no dissatisfaction)            & 83.4 & 89.2 & $+5.8$ \\
low (mild)                            & 13.7 &  9.6 & $-4.1$ \\
mid + high (meaningful)               &  2.9 &  1.3 & $-1.6$ \\
\midrule
\textit{Any dissatisfaction (low+mid+high)} & \textit{16.6} & \textit{10.8} & \textit{$-5.8$} \\
\bottomrule
\end{tabular}
\caption{Next-turn dissatisfaction rates for matched Computer vs.\ Search responses among sessions with multiple turns. Each response is scored on the scale \{zero, low, mid, high\} based on what the user does next (re-asks, corrections, error reports, retries, etc.). Computer's autonomous responses elicit measurably lower dissatisfaction than Search's chat responses across every level of the signal. Columns may not sum exactly to 100\% due to rounding.}
\label{tab:user_satisfaction}
\end{table}

\section{Efficiency}\label{efficiency}

Computer's autonomous execution shifts human labor from manual work to oversight. To quantify this shift, we estimate the human and compute time and cost required to complete each matched task under two regimes: (1)~Search~+~Human, where Search handles information retrieval and processing but the human must manually perform all non-research actions\footnote{Here, ``manual'' means no access to other AI tools; use of standard software is still permitted.}, and (2)~Computer~+~Human, where the model executes the full workflow and the human provides only oversight. We also test the assumptions and lemmas in Section~\ref{framework}.

\subsection{Method}
It can be challenging to estimate the human time spent on a task in a field setting, so we triangulate via three independent approaches: tool-based estimate, LLM-based estimate, and user self-reports from interviews. The first two approaches are based on production data.

\paragraph{Tool-based estimate.}
We classify Computer's tool calls into two categories based on what a Search user would still need to do manually (Table~\ref{tab:tool_classification}):

\begin{itemize}
    \item \textbf{``Search'' tools}: research actions that Search already handles. These include \texttt{search\_web}, \texttt{fetch\_url}, \texttt{search\_vertical}, \texttt{navigate} (browsing), \texttt{read} (document analysis), and result-submission tools (\texttt{submit\_answer}, \texttt{submit\_analysis}, \texttt{submit\_result}).
    \item \textbf{``do'' tools}: actions the human must execute themselves, since Search provides only information. These include, for instance, \texttt{bash} (terminal commands), \texttt{write} (create files), \texttt{edit} (modify files), \texttt{js\_repl} (run code), \texttt{computer} in browser-task mode (web app interaction), \texttt{call\_external\_tool} (connector calls), and \texttt{share\_file} (export deliverables).
\end{itemize}

For each of the 10{,}000 matched Computer sessions, we sum the estimated human-equivalent minutes from ``do'' tool calls (Table~\ref{tab:tool_classification}). This represents the manual work a Search user would still need to perform after receiving Search's answer. Computer's human time is fixed at 10~minutes of oversight per task (e.g., writing the prompt and reviewing the output). For cost, we combine model cost with human labor cost. Human labor cost uses Bureau of Labor Statistics Occupational Employment and Wage Statistics (BLS OEWS, May 2025, the most recent available) mean hourly wages mapped to each query domain (Table~\ref{tab:wages}).

\begin{table}[!htbp]
\centering
\small
\begin{tabular}{llr}
\toprule
Tool & Manual equivalent for a Search user & Min/call \\
\midrule
\textit{``Search'' tools: already covered by Search} & & \\
\quad \texttt{search\_web}        & Run a web search                                  & 0 \\
\quad \texttt{fetch\_url}         & Open a URL and read its contents                  & 0 \\
\quad \texttt{search\_vertical}   & Query a domain-specific index (news, academic)    & 0 \\
\quad \texttt{navigate}           & Click through linked pages                        & 0 \\
\quad \texttt{read}               & Read and extract from a document                  & 0 \\
\quad \texttt{submit\_*}          & Compose the final written answer                  & 0 \\
\addlinespace
\textit{``Do'' tools: the human must execute} & & \\
\quad \texttt{bash}               & Run a terminal command                            & 5 \\
\quad \texttt{write}              & Author a new file from scratch                    & 15 \\
\quad \texttt{edit}               & Locate and modify text in an existing file        & 10 \\
\quad \texttt{js\_repl}           & Write and execute a code snippet                  & 15 \\
\quad \texttt{computer} (browser) & Click/type in a web application                   & 0.5 \\
\quad \texttt{browser\_task}      & Complete a multi-step browser workflow            & 10 \\
\quad \texttt{call\_external\_tool} & Issue an API call (auth, request, parse)        & 5 \\
\quad \texttt{deploy\_website}    & Deploy code to a host or server                   & 5 \\
\quad \texttt{start\_server}      & Spin up a local dev server or service             & 5 \\
\quad \texttt{find}               & Locate a file or string in a project              & 3 \\
\quad \texttt{system\_diagnostic} & Inspect system state, check logs                  & 3 \\
\quad \texttt{share\_file}        & Export and deliver a file                         & 2 \\
\bottomrule
\end{tabular}
\caption{Tool classification used in the tool-based estimate. ``Search'' tools mirror capabilities Search already provides, so no manual time is charged. ``Do'' tools require the human counterfactual user to act on Search's research output. Per-call minute estimates approximate the time a skilled professional would spend performing the equivalent action by hand.}
\label{tab:tool_classification}
\end{table}

\begin{table}[!htbp]
\centering
\small
\begin{tabular}{llr}
\toprule
BLS occupation group & Query domains mapped & Wage (\$/hr) \\
\midrule
Legal                 & Law                                          & 67 \\
Software \& IT        & Technology, Programming                      & 58 \\
Healthcare            & Health                                       & 52 \\
Business              & Business                                     & 46 \\
Finance \& Investing  & Finance                                      & 46 \\
Science \& Research   & Science                                      & 45 \\
Writing \& Research   & Writing                                      & 38 \\
Media \& Creative     & Entertainment                                & 38 \\
Social Media          & Social Media                                 & 38 \\
Education             & Education                                    & 32 \\
Career \& Jobs        & Careers                                      & 30 \\
Government \& Policy  & Politics                                     & 30 \\
Personal \& Lifestyle & Travel, Consumer Goods, Local                & 25 \\
Real Estate           & Real Estate                                  & 25 \\
Other                 & Other                                        & 34 \\
\bottomrule
\end{tabular}
\caption{Hourly wage estimates used for cost calculations. Wages are BLS OEWS mean hourly wages (rounded to whole dollars), May 2025 release, mapped from BLS major occupation groups to the 18 query domains. The same domain-specific wage is applied to human time spent with Search and Computer. Rows are sorted by hourly wage (descending), with Other placed last.}
\label{tab:wages}
\end{table}

\paragraph{LLM-based estimate.} The tool-time mapping relies on two key assumptions: 1) accurate per-tool human-equivalent time estimates and 2) humans follow the same procedure as Computer. As an independent validation, we estimate the total human time required for each of the 10{,}000 matched pairs, given only the query text. The prompt describes the Search~+~Human counterfactual: a skilled professional receives research answers from Search but must perform all execution steps manually. The LLM returns a time estimate per task.

\paragraph{User-reported estimate.} Both production-data approaches fix the counterfactual to Search~+~Human, which may not capture all realistic pre-agent workflows. To capture a richer counterfactual, we conducted 45-minute semi-structured interviews with 25 active Computer users (6~enterprise, 19~consumer), recruited from those with at least 5 historical queries. Participants walked through specific completed tasks, described their pre-Computer workflow for each, and estimated the time that workflow would have taken. Self-reports are subject to recall and selection bias, but unlike the matched-pair estimates they reflect the counterfactual users actually would have chosen rather than the Search~+~Human baseline. These interviews are summarized by theme in Appendix~\ref{app:highlights}.

\subsection{Results}

\paragraph{Tool-based estimates.} Using the tool-based approach, Computer reduces both time and cost substantially across all 18 domains (Figure~\ref{fig:efficiency}, Table~\ref{tab:domain_multipliers}). The mean Search~+~Human task requires 121--596~minutes of manual execution (269~minutes overall), compared to 25--48~minutes for Computer~+~Human (36~minutes overall). Computer saves 79--92\% of task time (overall 87\%), with reductions largest for the most labor-intensive domains such as Programming (596 vs.\ 48~min, 92\% saved), Technology (280 vs.\ 37~min, 87\%), and Social Media (224 vs.\ 34~min, 85\%). The cost savings are similarly large: 87--96\% (overall 94\%). Cost savings exceed time savings because domain-specific wages amplify the effect: in high-wage domains like Programming (\$58/hr), a 92\% time reduction yields a 96\% cost reduction. Despite Computer's higher model costs (\$2--13 per task versus \$0.05 for Search), human labor dominates: model costs account for less than 0.1\% of the Search~+~Human total, while they constitute 31--59\% of Computer~+~Human's cost.

\begin{figure}[!htbp]
    \centering
    \includegraphics[width=\textwidth]{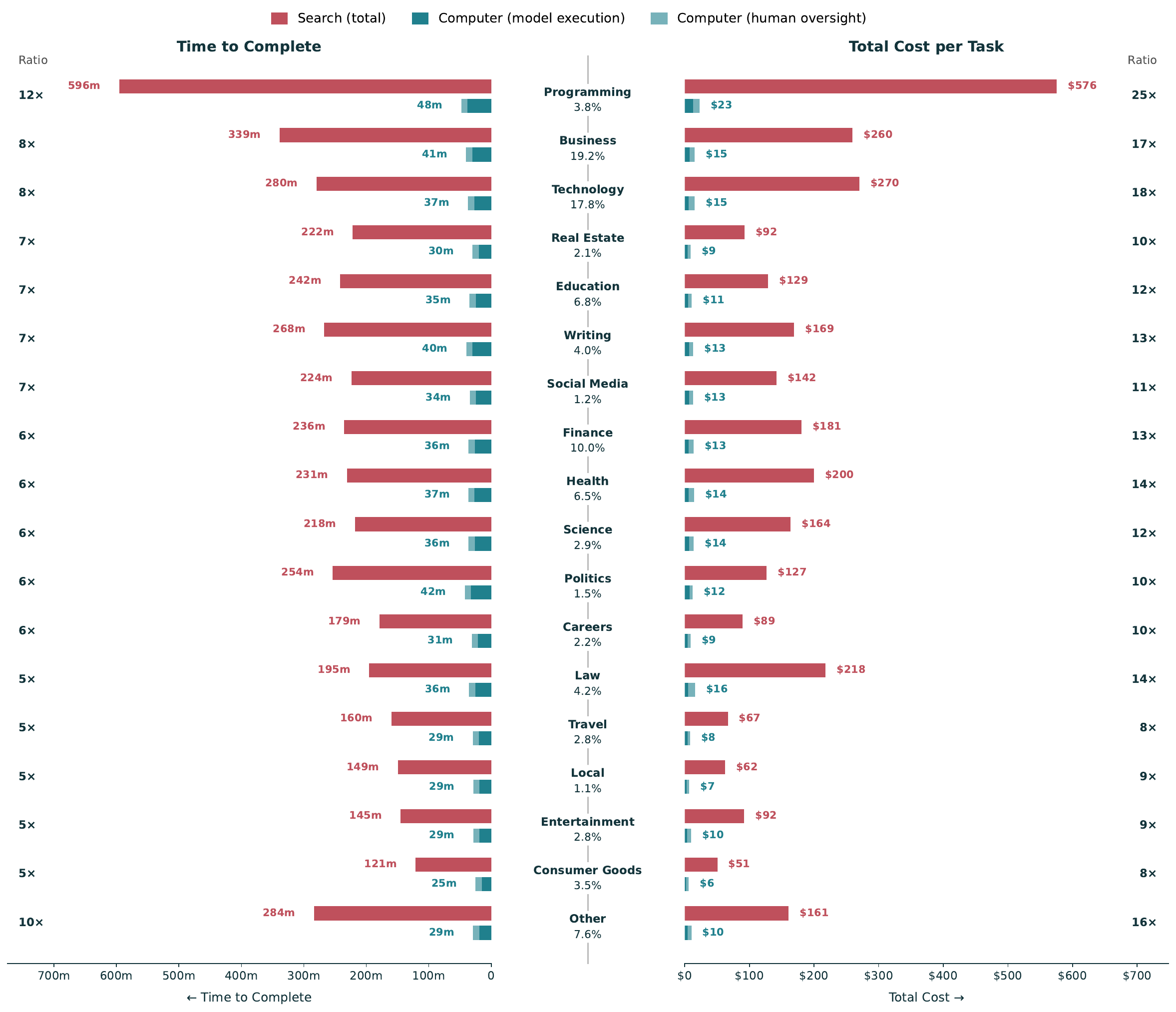}
    \caption{Task efficiency: time and cost for Search~+~Human vs.\ Computer~+~Human. Search~+~Human is dominated by human time and cost so the combined total is reported. \textbf{Left:} total time. \textbf{Right:} total cost. Despite higher model costs, Computer~+~Human saves 87--96\% on total cost because it shifts expensive human execution toward machine computation and oversight. Domains are sorted by the Search~+~Human / Computer~+~Human time ratio (descending), with Other placed last. Ratios are computed from unrounded means
and may differ from those implied by the rounded times and costs shown.}
    \label{fig:efficiency}
\end{figure}

\paragraph{Breakeven and sensitivity analysis.} How fast would a human need to perform the execution steps for Search~+~Human to match Computer~+~Human's total cost? Setting the two costs equal yields a breakeven threshold of 14--24~minutes across all domains (median 18~minutes). A professional would need to run all commands, edit all files, and navigate all web applications in under $20$~minutes to match Computer~+~Human's total cost.

The result is robust to two human time assumptions. First, even if our per-tool human-equivalent time estimates are overstated by $16\times$, Computer~+~Human retains a cost advantage on average ($8\times$ in the tightest domain). Second, the 10-minute human oversight assumption can be inflated by $26\times$ (to 260~minutes) before Computer~+~Human's cost advantage disappears on average ($12\times$ in the tightest domain). Figure~\ref{fig:cost_sensitivity} in Appendix~\ref{app:cost_sensitivity} plots how the cost advantage erodes as we make the assumptions increasingly conservative. The time advantage is robust to the same assumptions, though it is tighter than the cost advantage: it survives per-tool overstatement to $7\times$ overall ($5\times$ in the tightest domain) and oversight inflation to $24\times$ overall ($11\times$ in the tightest domain), as shown in Figure~\ref{fig:time_sensitivity} in Appendix~\ref{app:time_sensitivity}.

\paragraph{LLM-based estimate.} The LLM approach confirms the tool-based findings (Table~\ref{tab:domain_multipliers}). The LLM-based estimates yield similar advantages: 79--88\% time saved (overall 84\%) and 88--94\% cost saved (overall 93\%) across all 18 domains. The two methods yield task-level estimates of comparable magnitude: mean human time is 269~minutes under the tool-based approach and 227~minutes under the LLM-based approach. LLM estimates include time costs the tool mapping cannot measure, such as planning and digesting Search responses. The tool mapping might also be sensitive to which tools each domain's tasks invoke, while the LLM judges from query text alone. Overall, the two approaches agree that Computer~+~Human delivers large savings across every domain: tool-based 79--92\% time and 87--96\% cost; LLM 79--88\% time and 88--94\% cost.

\begin{table}[!htbp]
\centering
\small
\begin{tabular}{lrcccc}
\toprule
 & & \multicolumn{2}{c}{Tool-based} & \multicolumn{2}{c}{LLM-based} \\
\cmidrule(lr){3-4} \cmidrule(lr){5-6}
Task Domain & $n$ & Time Save (\%/$\times$) & Cost Save (\%/$\times$) & Time Save (\%/$\times$) & Cost Save (\%/$\times$) \\
\midrule
Programming & 378 & 92 (12) & 96 (25) & 81 (5.4) & 91 (11) \\
Business & 1{,}922 & 88 (8.4) & 94 (17) & 84 (6.3) & 92 (13) \\
Technology & 1{,}782 & 87 (7.5) & 94 (18) & 83 (5.8) & 93 (14) \\
Real Estate & 210 & 86 (7.4) & 90 (10) & 88 (8.2) & 91 (12) \\
Education & 677 & 86 (6.9) & 92 (12) & 85 (6.5) & 91 (11) \\
Writing & 400 & 85 (6.7) & 92 (13) & 81 (5.2) & 90 (10) \\
Social Media & 120 & 85 (6.5) & 91 (11) & 79 (4.8) & 88 (8.1) \\
Finance & 999 & 85 (6.5) & 93 (13) & 85 (6.5) & 93 (13) \\
Health & 654 & 84 (6.3) & 93 (14) & 86 (7.0) & 94 (15) \\
Science & 295 & 83 (6.0) & 91 (12) & 88 (8.5) & 94 (16) \\
Politics & 145 & 83 (6.0) & 90 (10) & 83 (6.0) & 90 (10) \\
Careers & 220 & 82 (5.7) & 90 (9.8) & 85 (6.5) & 91 (11) \\
Law & 420 & 82 (5.5) & 93 (14) & 86 (7.3) & 94 (18) \\
Travel & 277 & 82 (5.4) & 87 (8.0) & 85 (6.9) & 90 (10) \\
Local & 111 & 81 (5.2) & 89 (8.9) & 84 (6.4) & 91 (11) \\
Entertainment & 276 & 80 (5.0) & 89 (9.5) & 84 (6.4) & 92 (12) \\
Consumer Goods & 353 & 79 (4.8) & 88 (8.0) & 83 (5.8) & 90 (9.7) \\
Other & 761 & 90 (9.7) & 94 (16) & 82 (5.6) & 89 (9.1) \\
\midrule
\textbf{Overall} & 10{,}000 & 87 (7.4) & 94 (16) & 84 (6.3) & 93 (13) \\
\bottomrule
\end{tabular}
\caption{Percentage of time and cost saved by Computer~+~Human relative to Search~+~Human, with multipliers in parentheses (e.g., 94\% (16$\times$) means Computer~+~Human is 94\% or 16 times cheaper). Human labor cost uses BLS OEWS May 2025 mean hourly wages. Rows are sorted by tool-based time savings (descending), with Other placed last. Percentages and multipliers are computed from unrounded values and may differ from those implied by the rounded figures shown.}
\label{tab:domain_multipliers}
\end{table}

\paragraph{User-reported estimates.} As a third validation, we draw on semi-structured interviews with 25 Computer users (6~enterprise, 19~consumer) who were asked to describe specific tasks and estimate time saved. Among participants who provided quantifiable before/after comparisons, self-reported speedups range from $5\times$ to over $300\times$, likely reflecting substantial variation in pre-Computer baselines, with a per-participant median of approximately $25\times$. Representative examples are summarized in Appendix~\ref{app:highlights_efficiency}.\\

\noindent Next, we connect the cost estimates to Assumptions~\ref{asmp:fixed} and~\ref{asmp:marginal} and Lemma~\ref{lem:sorting} in Section~\ref{framework}. We proxy step count $s$ by the total number of Computer tool calls (both ``Search'' and ``do'' tools) per task.

\paragraph{Fixed cost.} We proxy the per-task fixed cost by total query characters per session (summing the initial query and any user-issued follow-ups), which captures prompt-writing effort across the full task rather than only the opening turn.\footnote{This approach does not capture the verification cost for Computer outputs so its fixed cost might be underestimated.} We compare within matched pairs to hold task content fixed; Computer sessions are about $46\%$ longer than Search sessions at the median ($652$ vs.\ $448$ characters). This gap is consistent with Assumption~\ref{asmp:fixed} ($f_{\Agent} > f_{\Conversational} > 0$): delegating a whole workflow requires more up-front scoping than issuing a one-shot question.

\paragraph{Marginal cost.} Across the matched pairs, total variable cost divided by total tool calls yields \$0.16 per step for Computer~+~Human and \$2.05 per step for Search~+~Human, a $13\times$ gap. The same pattern holds in time: total task execution minutes divided by total tool calls yields $0.46$ minutes per step for Computer~+~Human and $2.66$ minutes per step for Search~+~Human, a $6\times$ gap. Both comparisons support Assumption~\ref{asmp:marginal} ($m_{\Conversational} > m_{\Agent} > 0$).

\paragraph{How Computer's efficiency advantage scales with task steps.} To trace how the gap scales with task complexity, we regress the within-pair wedge between Search~+~Human and Computer~+~Human on $N_{\text{total}}$ tool calls, separately for time and for cost (Table~\ref{tab:cost_regressions}). Both slopes are positive: each additional step widens the time gap by $2.39$ minutes and the cost gap by $\$1.94$. This is consistent with the cost structure underlying Lemma~\ref{lem:sorting}: because the wedge is strictly increasing in~$s$, longer tasks reap larger savings under Computer~+~Human, the mechanism that in the model generates the sorting threshold $s^{\ast} = (f_{\Agent} - f_{\Conversational})/(m_{\Conversational} - m_{\Agent})$.\footnote{Using independent LLM-based estimates of time and cost to construct the wedge yields consistently positive and significant results.}

\begin{table}[!htbp]
\centering
\small
\begin{tabular}{lcc}
\toprule
                          & (1)                & (2) \\
                          & Time wedge (min)   & Cost wedge (\$) \\
\midrule
$s = N_{\text{total}}$    & $2.39^{***}$       & $1.94^{***}$ \\
                          & $(0.065)$          & $(0.057)$ \\
\addlinespace
Observations              & 10{,}000            & 10{,}000 \\
$R^{2}$                   & 0.72               & 0.68 \\
\bottomrule
\end{tabular}
\caption{Within-pair wedges (Search~+~Human minus Computer~+~Human) regressed on total Computer tool-call count $s = N_{\text{total}}$ across matched pairs. Each column reports a separate regression: column~(1) the time wedge in minutes; column~(2) the cost wedge in dollars. Heteroskedasticity-robust standard errors in parentheses. Significance: $^{*}\,p<0.10$; $^{**}\,p<0.05$; $^{***}\,p<0.01$.}
\label{tab:cost_regressions}
\end{table}

\section{Scope}\label{scope}

The previous sections use within-user, within-task comparison to show Computer's autonomous execution saves time and cost. Now we turn to within-user, cross-task analysis to test if Computer also enables users to expand their work scopes at the query level. Because Computer sessions contain more queries on average, session-level estimates further amplify the query-level differences. We examine two dimensions: \emph{horizontal expansion} (whether users work across multiple occupations) and \emph{vertical expansion} (whether users attempt more complex tasks).

\subsection{Horizontal expansion: cross-occupation work}

\subsubsection{Method}

\paragraph{Occupation inference.}

We assign Computer users to one of the 8 most common occupation clusters in the data based on their pre-Computer Search behavior. The clusters are based on the National Career Clusters Framework used in O*NET.\footnote{We drop Hospitality/Events/Tourism, Advanced Manufacturing, Energy \& Natural Resources, Construction,
  Supply Chain \& Transportation, and Agriculture because these clusters cover primarily physical,
  hands-on work that is largely outside the scope of digital knowledge work and accounts for negligible
  query volume on both products.} For each user, we map the topic domain and subdomain of each Search query to a standardized occupation cluster (e.g., Programming $\to$ Digital Technology; Finance $\to$ Financial Services; Health $\to$ Healthcare). Multi-cluster domains such as Business are split by subdomain (e.g., Business/Marketing $\to$ Marketing \& Sales; Business/Management $\to$ Management \& Entrepreneurship). The user's occupation is defined as the cluster containing their most frequently queried domain and subdomain. To ensure within-user comparison, we restrict the sample to users who actively use both Computer and Search during the study period. From each occupation cluster, we draw 1{,}000 dual-product users, yielding 8{,}000 users in total. Computer queries are then mapped to a destination cluster directly via an LLM single-label classification into the same 8-cluster taxonomy.\footnote{For both Computer and Search queries, those that do not belong to any of these clusters are classified as ``Other'' and excluded.}

\paragraph{Cross-occupation task share.} For each user, we compute the fraction of their queries that fall outside their primary occupation cluster. Each Computer and Search query is mapped to an occupation cluster and compared to the user's assigned cluster.

\subsubsection{Results}

\paragraph{Cross-occupation task share.} Computer users consistently work outside their primary occupation at higher rates than when using Search (Figure~\ref{fig:scope}). Across the 8 occupation clusters, Computer's cross-occupation share averages 59\%, compared to 50\% for Search, a 9~pp increase. The effect is largest for Management \& Entrepreneurship (+19~pp), Digital Technology (+13~pp), Arts \& Design (+12~pp), and Healthcare \& Human Services (+10~pp). Education (+6~pp) and Marketing \& Sales (+5~pp) shift modestly, while Public Service \& Safety (+2~pp) and Financial Services (+1~pp) are nearly flat. The pattern is also corroborated by interviews in Appendix~\ref{app:highlights_scope}.

\begin{figure}[!htbp]
    \centering
    \includegraphics[width=0.75\textwidth]{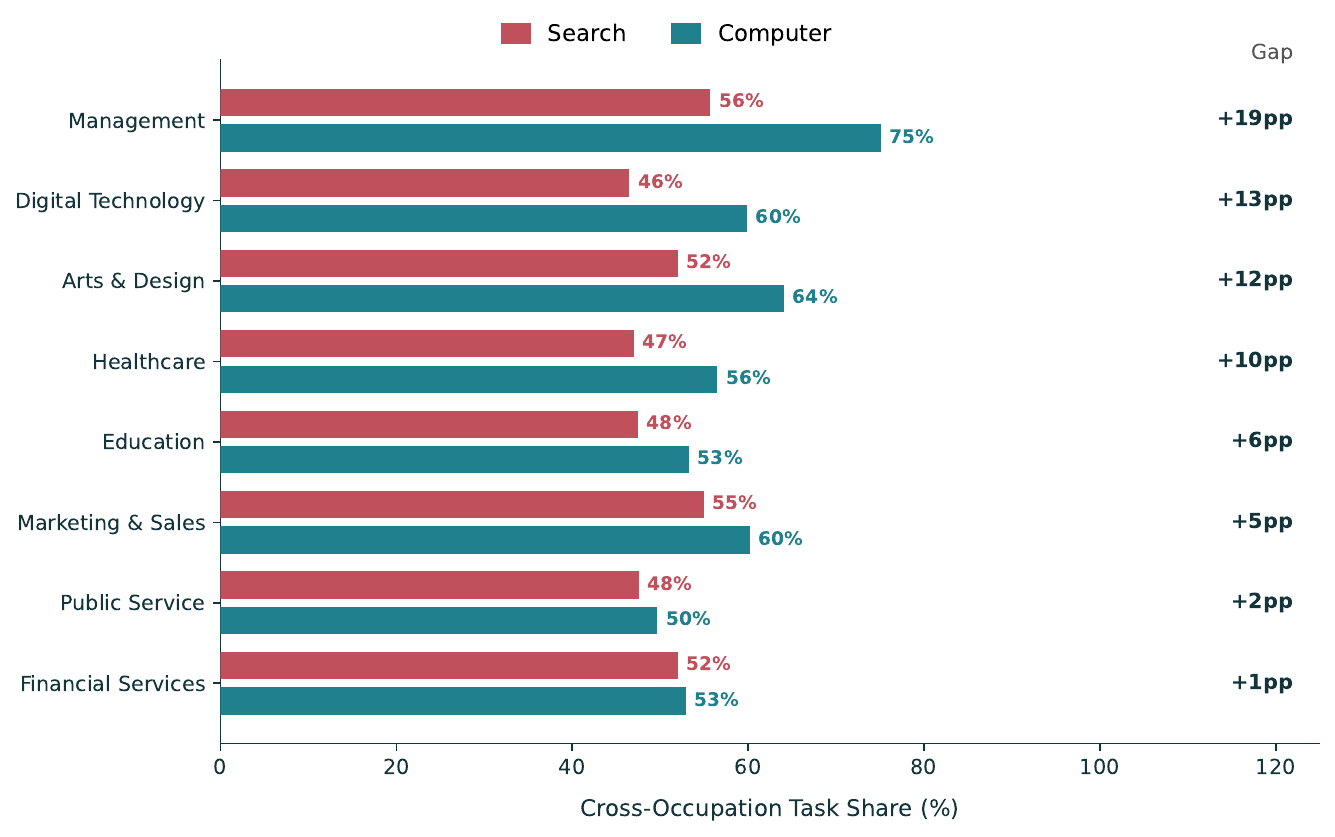}
    \caption{Cross-occupation task share by inferred occupation. Bars are averaged across users within each primary cluster. Computer users work outside their primary occupation more frequently than when using Search, particularly in Management (+19~pp), Digital Technology (+13~pp), Arts \& Design (+12~pp), and Healthcare (+10~pp). Clusters are sorted by the Computer$-$Search gap in pp (descending). Gaps are computed from unrounded shares and may differ from the difference of the rounded values shown.}
    \label{fig:scope}
\end{figure}

\paragraph{Cross-occupation destinations.} Figures~\ref{fig:cross_occ_heatmap} and~\ref{fig:cross_occ_flow} break down the share of each cluster's queries directed at occupations other than the user's own, and the two products route this work differently. For Search, Digital Technology acts as a universal sink: it is the top out-of-occupation destination for all clusters except itself, consistent with a workflow dominated by ubiquitous technical lookups. For Computer, no single hub dominates. Digital Technology remains the leading out-of-occupation destination for four clusters---Management (17\%), Financial Services (15\%), Marketing (13\%), and Arts \& Design (13\%)---but its pull is weaker than on the Search side, as reflected in the smaller share gap between the top and second destination. The flows beyond it are diverse: Management users send nearly as much work to Public Service and Financial Services (14\% each) as to Digital Technology, Digital Technology users delegate outward to Financial Services (12\%), Education (10\%), and Public Service (10\%), and Healthcare and Education exchange work in both directions (12\% each way). Figure~\ref{fig:cross_occ_flow} renders the contrast directly: a single dominant hub on the Search side, a web of flows on the Computer side. This suggests that Search cross-occupation queries look up technical facts, whereas Computer cross-occupation queries delegate work that would normally require a specialist in a different field.

\begin{figure}[!htbp]
    \centering
    \includegraphics[width=\textwidth]{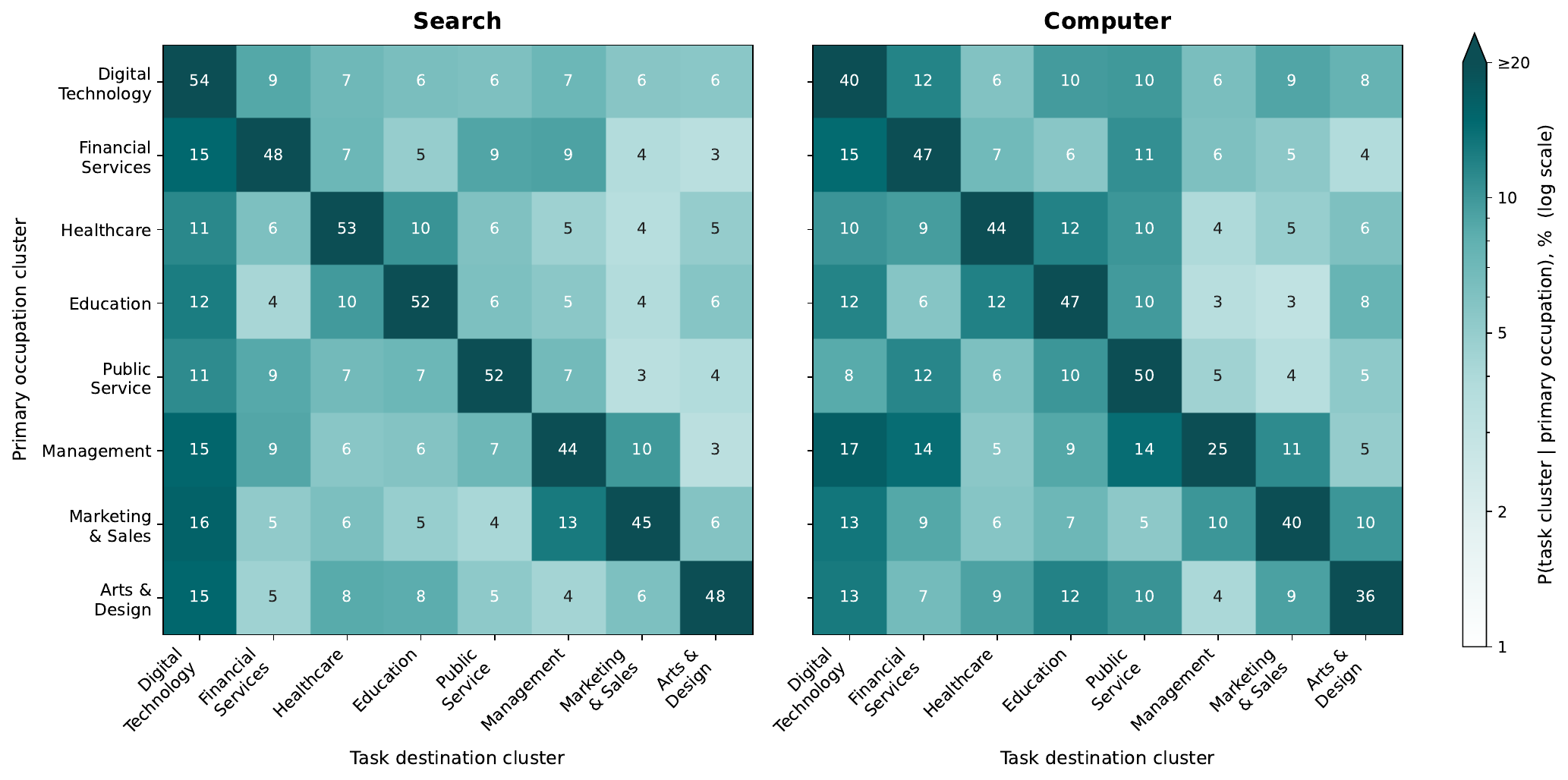}
    \caption{Cross-occupation task destinations: $P(\text{task cluster} \mid \text{primary occupation})$ for Search (left) and Computer (right) across the 8 occupation clusters. Each row sums to 100\%, subject to rounding; the diagonal is the share of work that stays within the user's primary occupation. Cells are averaged across users within each primary cluster. Color intensity is on a log scale. Computer's diagonal is consistently weaker than Search's, and its off-diagonal mass is spread more widely across multiple destination clusters.}
    \label{fig:cross_occ_heatmap}
\end{figure}

\begin{figure}[!htbp]
    \centering
    \includegraphics[width=\textwidth]{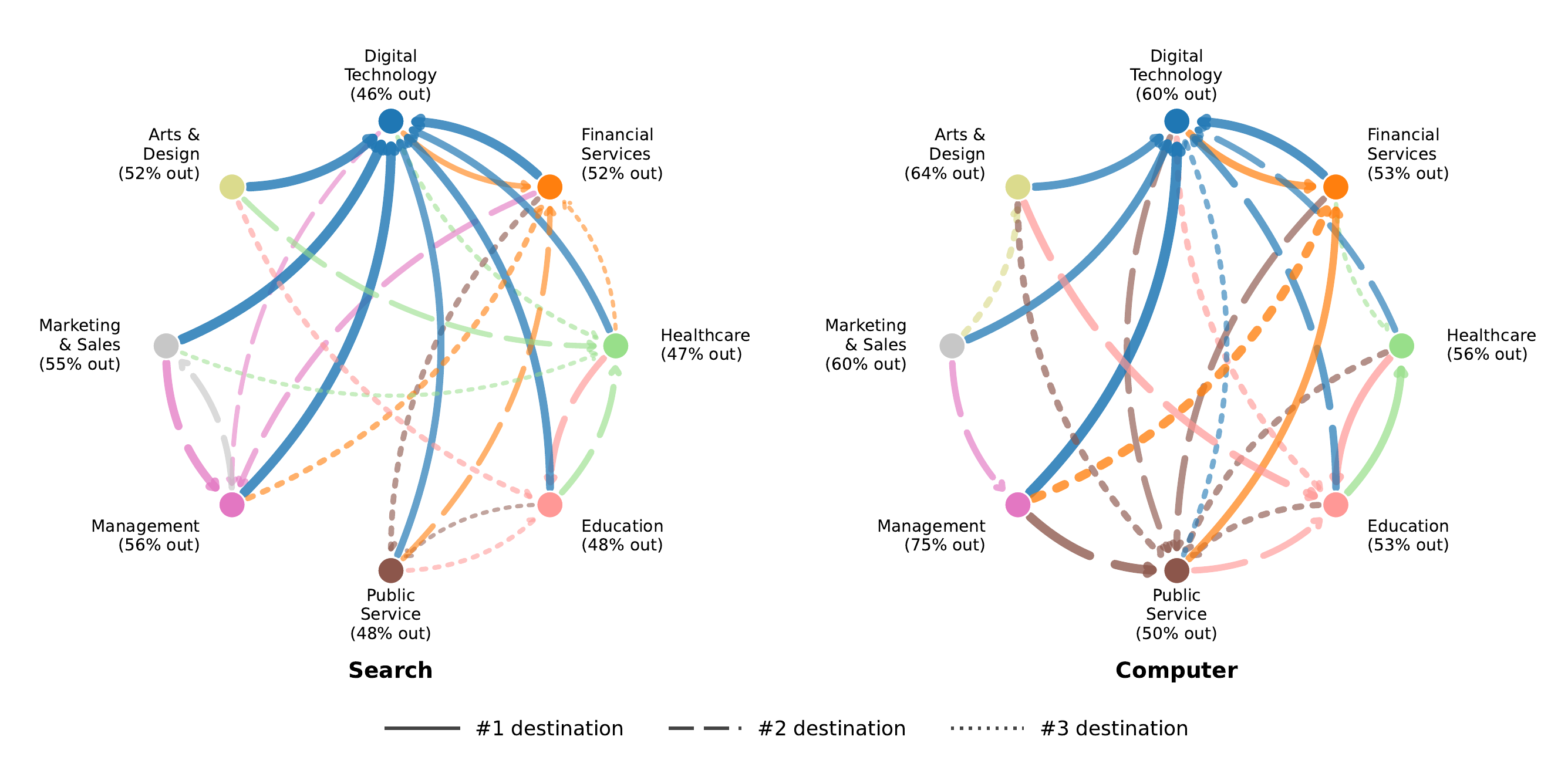}
    \caption{Cross-occupation task flows on a ring of the 8 occupation clusters. Each node shows the share of that cluster's tasks that are out of the user's primary occupation; curved arcs show each cluster's top-3 outgoing destinations, with line width proportional to $P(\text{cluster }Y \mid \text{primary}=X)$ for $Y \neq X$ and line type indicating the rank. Lines are colored by the destination cluster. \textbf{Left:} Search concentrates flows into Digital Technology as the universal top destination. \textbf{Right:} Computer spreads flows across Digital Technology, Public Service, Financial Services, and Education.}
    \label{fig:cross_occ_flow}
\end{figure}

\subsection{Vertical expansion: task complexity}

The horizontal analysis above shows that Computer usage spans more occupations. Another question is whether Computer queries also attempt more demanding tasks. We measure task complexity along four axes: \emph{cognitive level}, \emph{knowledge breadth}, \emph{task composability}, and \emph{new tasks unlocked}.

\subsubsection{Method}

We sample 10{,}000 queries (5{,}000 initial Computer queries and 5{,}000 initial Search queries) from a fixed set of 5{,}000 dual-product users (one Computer query and one Search query per user, both drawn at random from each user's sessions in the study period). The Computer side is gated to sessions that invoked at least one execution (``do'') tool. Each query is classified by an LLM along four axes using established taxonomies:

\begin{itemize}
    \item \textbf{Cognitive level: Bloom's Revised Taxonomy} \citep{anderson2001taxonomy}: six levels of cognitive demand, from \emph{Remember} (factual recall) through \emph{Understand}, \emph{Apply}, \emph{Analyze}, \emph{Evaluate}, to \emph{Create} (producing novel output). We group these into lower-order (Remember, Understand, Apply) and higher-order (Analyze, Evaluate, Create). Each query is assigned to the highest level required to complete the request.
    \item \textbf{Cognitive level: Task type} \citep{autor2003skill}: \emph{abstract} (non-routine cognitive work requiring judgment, creativity, or strategic reasoning) versus \emph{routine} (rule-following tasks such as fact lookup, format conversion, or template-based writing).
    \item \textbf{Knowledge breadth}: for each query, the minimal set of O*NET Knowledge areas \citep{onet} whose substantive expertise is required to complete the task well. O*NET, maintained by the U.S.~Department of Labor, defines 33 Knowledge areas (e.g., Economics and Accounting, Design, Medicine and Dentistry, Law and Government) that are used to characterize the knowledge requirements of nearly every U.S.~occupation. The classifier is instructed to return a \emph{minimal} set: a domain counts only if doing the task well requires real knowledge in that area, not if the topic merely appears.\footnote{For instance, a query about ``the 2008 financial crisis'' does not require Economics expertise if it is a simple lookup, but ``analyze this company's 10-K and flag accounting irregularities'' does.} For each query we record the list of domains and its cardinality (the \emph{breadth} metric).
    \item \textbf{Task composability}: each query is classified against the O*NET activity hierarchy at four nesting levels: 37 \emph{Generalized Work Activities} (GWAs, the broadest top-level groupings such as ``Getting Information'' or ``Thinking Creatively''); 332 \emph{Intermediate Work Activities} (IWAs, occupation-agnostic behaviors such as ``Analyze business or financial data'' or ``Create visual designs or displays''); 2{,}087 \emph{Detailed Work Activities} (DWAs, e.g., ``Analyze financial data to detect irregularities or trends''); and 18{,}796 occupation-specific \emph{Task Statements} (TSs, e.g., the Chief Executive task ``Direct or coordinate an organization's financial or budget activities''). To avoid mechanically inflating counts via parent-to-child rollouts, GWAs are classified directly against the 37-label catalog; IWAs are then classified only from the candidates implied by the chosen GWAs; and similarly for DWAs and TSs. For each query we record the engaged set at every level and its cardinality.
    \item \textbf{New tasks unlocked}: beyond counting activities per query, we treat each product's classified activity inventory as a set and ask which activities Computer attempts but Search essentially does not. At each O*NET nesting level (GWA, IWA, DWA, TS), the \emph{Computer-only set} at threshold $k$ is the set of activities with more than $k$ occurrences in the Computer queries and at most $k$ occurrences in the Search queries from the same users. The strictest case ($k=0$) corresponds to activities that appear in Computer but never in Search; relaxing $k$ admits activities Search attempts only rarely. For each Computer query we record whether it engages at least one Computer-only activity at each level.
\end{itemize}

\subsubsection{Results}

\paragraph{Cognitive level.} Computer queries are substantially more cognitively complex than Search queries from the same users (Figure~\ref{fig:complexity}; shares throughout are among queries classifiable on the relevant dimension). On Bloom's taxonomy, 76\% of Computer queries demand higher-order cognition (Analyze, Evaluate, or Create) compared to 55\% for Search, a 21~pp gap. The difference is concentrated at the top: 50\% of Computer queries are classified as \emph{Create} (producing novel artifacts such as code, reports, or designs) versus 26\% for Search. Conversely, Search is dominated by lower-order tasks: \emph{Remember}-level factual lookups account for 21\% of Search queries but only 7\% of Computer queries. On the Autor et al.\ task-type dimension, 71\% of Computer queries are \emph{abstract} (non-routine) versus 53\% for Search.

The concentration at \emph{Create}, rather than a uniform upward shift, suggests that autonomous execution specifically unlocks generative work (writing code, drafting documents, building artifacts) that users would not attempt through a question-answer interface. The middle levels of Bloom's taxonomy (Analyze, Evaluate) show roughly equal shares between products, indicating that the shift is not simply users rephrasing the same intent more ambitiously, but a qualitative change in the type of work delegated.

\begin{figure}[!htbp]
    \centering
    \includegraphics[width=\textwidth]{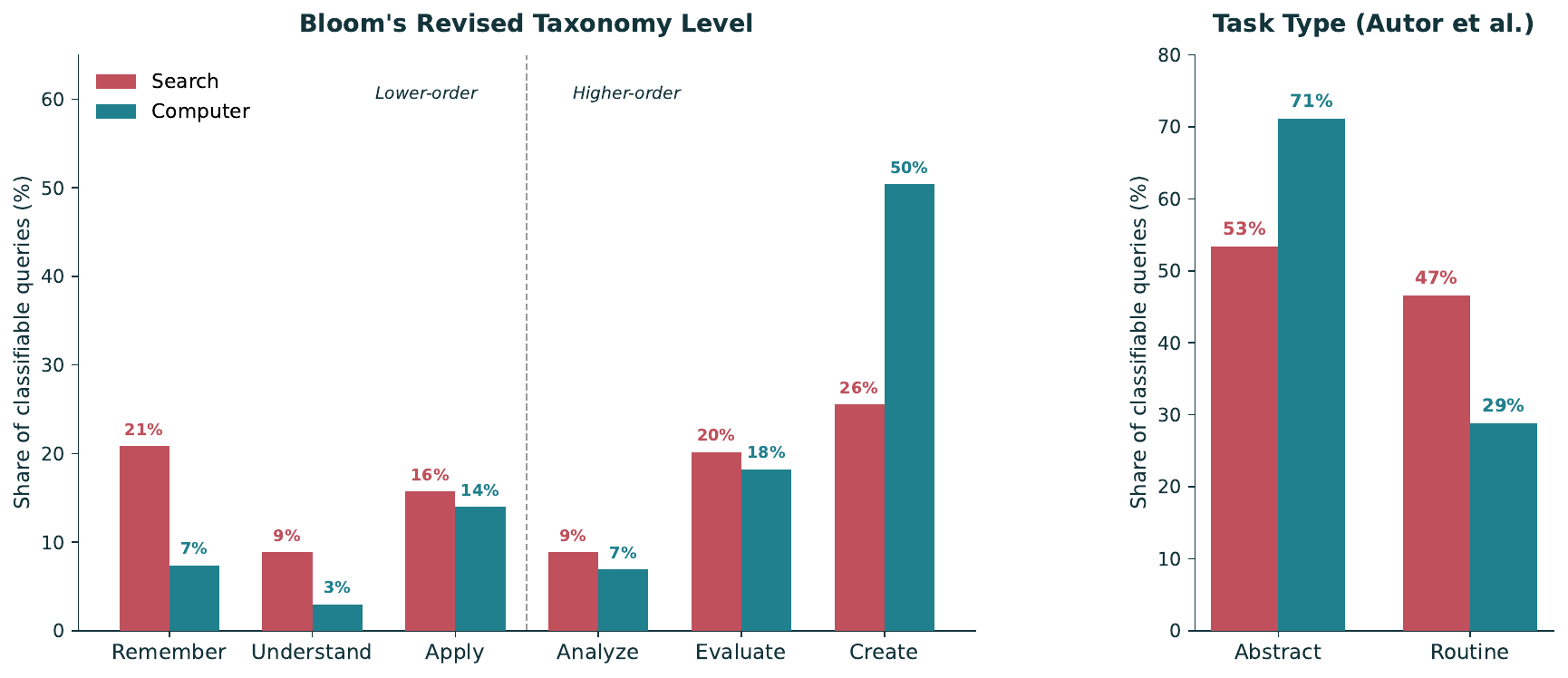}
    \caption{Cognitive complexity of Computer vs.\ Search queries. \textbf{Left:} Bloom's Revised Taxonomy distribution. Computer queries concentrate at \emph{Create} (50\% vs.\ 26\%); Search carries more lower-order weight at \emph{Remember} (21\% vs.\ 7\%). \textbf{Right:} Task type (Autor et al.). 71\% of Computer queries involve abstract, non-routine cognition vs.\ 53\% for Search. Shares are computed among queries classifiable on each dimension. Shares may not sum exactly to 100\% due to rounding. Gaps in text are computed from unrounded shares.}
    \label{fig:complexity}
\end{figure}

\paragraph{Knowledge breadth.} Beyond being more cognitively demanding, Computer tasks require substantive expertise in more distinct domains (Figure~\ref{fig:knowledge}). On average, a Computer task requires substantive expertise in 2.40 O*NET Knowledge domains, compared with 1.74 for Search, a 38\% increase. The shift is also a change in shape, not just location. Search queries cluster at 1--2 domains (77\% combined): the typical Search query is a focused lookup across few topics. Computer queries, by contrast, concentrate at 2--3 domains (76\% combined) and are nearly three times as likely as Search to require three or more domains (51\% vs.\ 17\%). These multi-competency queries are exactly the tasks that in a pre-agent workflow would require coordination across specialists, such as building a data-visualization dashboard for financial models (Design + Mathematics + Economics and Accounting).

The composition of which domains are invoked also differs sharply between products. Computer shows large prevalence gains in domains associated with production work: Design (+12~pp), Mathematics (+9~pp), Administration and Management (+9~pp), Computers and Electronics (+7~pp), Economics and Accounting (+6~pp), Communications and Media (+4~pp), and Sales and Marketing (+3~pp). Search's relative strongholds (Food Production, Sociology and Anthropology, Mechanical, Medicine and Dentistry, and Foreign Language) are precisely the domains where users tend to look up facts rather than execute on them.

\begin{figure}[!htbp]
    \centering
    \includegraphics[width=\textwidth]{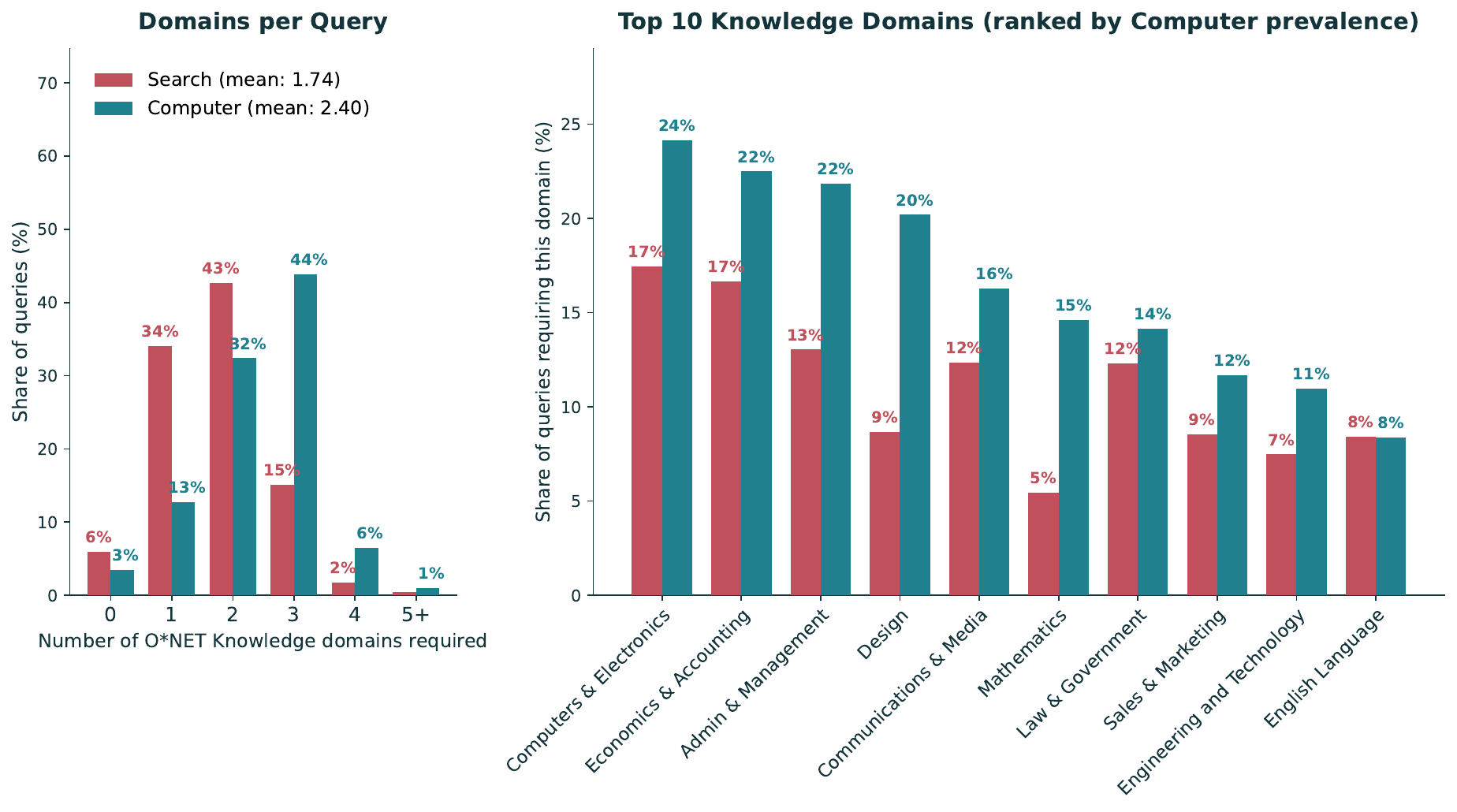}
    \caption{Required knowledge domains per query. \textbf{Left:} Distribution of the number of O*NET Knowledge areas a task requires substantive expertise in. Search concentrates at 1--2 domains (77\%); Computer queries concentrate at 2--3 (76\%) and are nearly 3$\times$ as likely to require $\geq 3$ domains (51\% vs.\ 17\%). Shares may not sum exactly to 100\% due to rounding. \textbf{Right:} Top 10 Knowledge domains by Computer prevalence, with Search shares side-by-side. Computer's largest gains are in executional/creative domains (Design, Mathematics, Administration and Management, Economics and Accounting, Computers and Electronics). Gaps in text are computed from unrounded shares.}
    \label{fig:knowledge}
\end{figure}

\paragraph{Task composability.} Computer queries decompose into more distinct work activities than Search queries from the same users, and the gap widens as the O*NET hierarchy is resolved more finely (Figure~\ref{fig:iwas}, Table~\ref{tab:task_levels}). At the coarsest GWA level, a typical Computer query engages 2.95 activities versus 2.24 for Search (+32\%), with 63\% of Computer queries engaging three or more GWAs versus 36\% of Search queries. At the IWA level the gap is larger: a typical Computer query engages 4.01 activities versus 2.87 for Search, a 40\% increase, and 83\% of Computer queries engage three or more IWAs versus 60\% of Search queries. Composition sharpens the information-versus-execution contrast at every level. At the GWA level, ``Getting Information'' is essentially shared (Computer 58\%, Search 56\%), while Computer's gains concentrate in production-oriented groupings: ``Documenting/Recording Information'' (+30~pp), ``Thinking Creatively'' (+24~pp), and ``Analyzing Data or Information'' (+14~pp). The same contrast appears among IWAs: ``Gather information from physical or electronic sources'' remains the most prevalent activity on both sides (Computer 46\%, Search 42\%), but Computer's gains concentrate in activities that produce artifacts or deliverables: ``Create visual designs or displays'' (+18~pp), ``Prepare informational or instructional materials'' (+16~pp), ``Prepare reports of operational or procedural activities'' (+13~pp), and ``Analyze business or financial data'' (+9~pp). The pattern strengthens at finer grains: at the DWA and TS levels, Computer queries engage $59\%$ and $60\%$ more activities than Search. The ten most prevalent activities at each level for both products are shown in Tables~\ref{tab:top10_gwa}--\ref{tab:top10_tasks} in Appendix~\ref{app:top10_activities}.

\begin{figure}[!htbp]
    \centering
    \includegraphics[width=.8\textwidth]{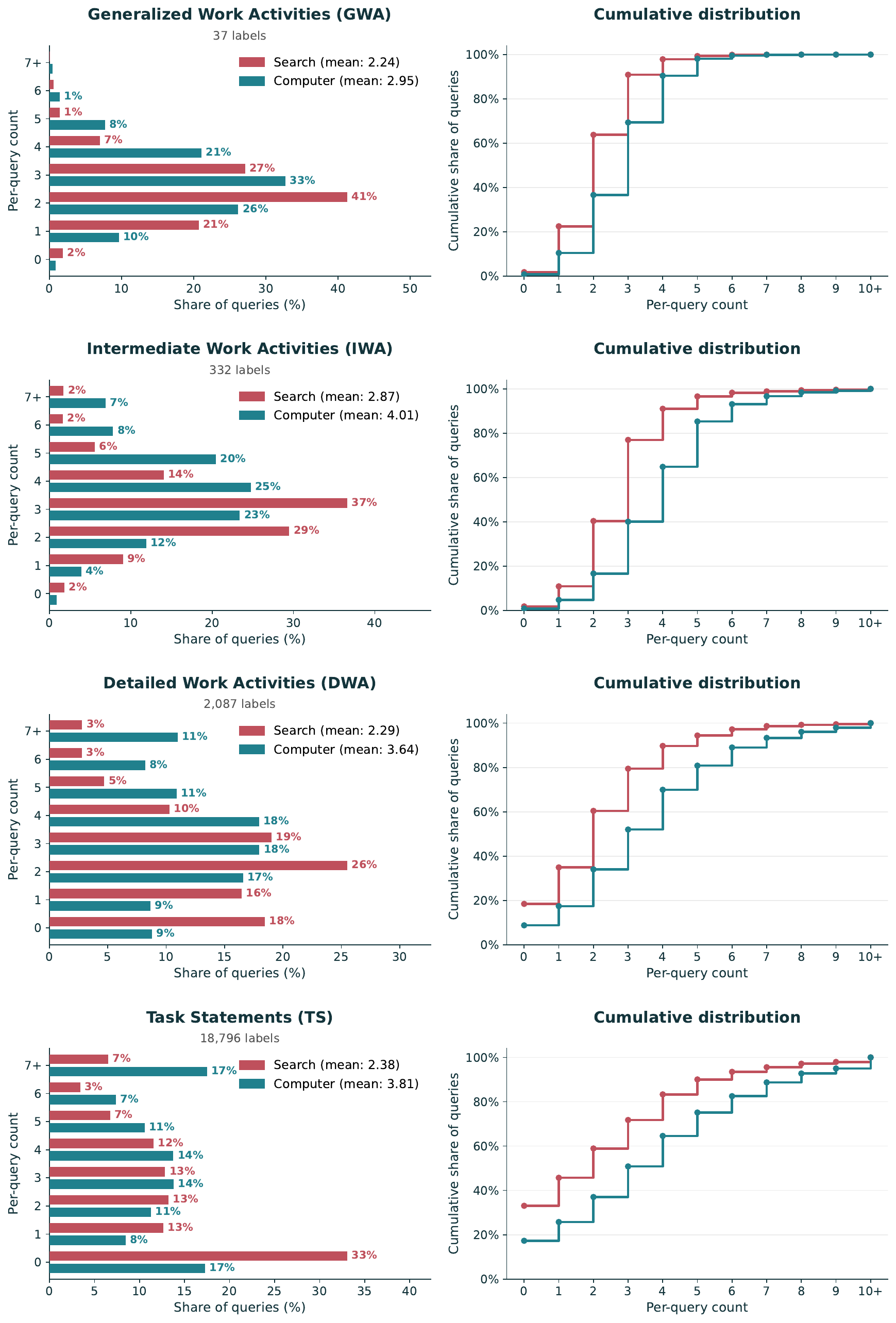}
    \caption{Per-query distribution of engaged O*NET task-level activities at four nesting depths: GWAs, IWAs, DWAs, and TSs. \textbf{Left column:} grouped-bar histogram of per-query count distributions. Shares may not sum exactly to 100\% due to rounding. \textbf{Right column:} corresponding empirical CDFs. Across all levels, Computer's distribution shifts right relative to Search. Search mass concentrates at $1$--$3$ activities while Computer concentrates at $3$--$5$ and beyond. Gaps in text are computed from unrounded shares.}
    \label{fig:iwas}
\end{figure}

\begin{table}[!htbp]
\centering
\small
\begin{tabular}{lrrrr}
\toprule
O*NET level & \# Labels & Search & Computer & Gap \\
\midrule
Generalized Work Activities (GWA)          &       37 & 2.24 & 2.95 & $+32\%$ \\
Intermediate Work Activities (IWA)             &    332   & 2.87 & 4.01 & $+40\%$ \\
Detailed Work Activities (DWA)             &  2{,}087 & 2.29 & 3.64 & $+59\%$ \\
Task Statements (TS) & 18{,}796 & 2.38 & 3.81 & $+60\%$ \\
\bottomrule
\end{tabular}
\caption{Task-activity breadth at four levels of the O*NET hierarchy. Search/Computer columns report per-query means; gap is $(\text{Computer}-\text{Search})/\text{Search}$, computed from unrounded means and may differ from those implied by the rounded values shown. The Computer$-$Search gap widens monotonically with granularity ($+32\%$ at the coarse GWA level to $+60\%$ at the fine TS level), indicating that Computer's distinctiveness lies in fine-grained executional work rather than coarse topical range.}
\label{tab:task_levels}
\end{table}

\paragraph{New tasks unlocked.} Proposition~\ref{prop:afford} shows that agent access weakly expands the affordable task frontier toward weakly higher-value tasks. Here, we measure new-task activity directly. We treat each product's classified activity inventory as a set: at threshold $k$, the \emph{Computer-only set} is the set of activities with more than $k$ occurrences in Computer queries and at most $k$ occurrences in Search queries from the same users.

Figure~\ref{fig:setdiff_curve} plots the share of Computer queries that engage at least one Computer-only activity, swept across $k = 0$--$10$ at the four O*NET nesting levels. At the strictest definition ($k=0$: activities Search never attempts in this sample), $23\%$ of Computer queries that engage at least one TS involve a Computer-only TS, $5\%$ engage a Computer-only DWA, under $1\%$ engage a Computer-only IWA, and effectively none engage a Computer-only GWA. Relaxing the Search ceiling to $k=5$ raises these to $38\%$, $18\%$, $2\%$, and under $1\%$ respectively. At the coarse GWA and IWA levels, Search and Computer cover nearly the same topical surface (the Computer-only share remains in the low single digits), but at the fine-grained TS level Computer's inventory is substantially larger than Search's, and a sizable share of Computer usage sits in the excess.

The Computer-only set at moderate thresholds concentrates in three capability clusters that align with the IWAs identified above: \emph{software and web development} (e.g., ``Develop or maintain internal or external company Web sites,'' ``Maintain or update business intelligence tools, databases, dashboards, systems, or methods''), \emph{documentation production} (e.g., ``Create or revise user instructions, procedures, or manuals,'' ``Prepare support documentation and training materials''), and \emph{data visualization and graphics} (e.g., ``Compile reports, charts, or graphs that describe and interpret findings of analyses,'' ``Process large amounts of data for statistical modeling and graphic analysis, using computers''). These are precisely the categories where Computer's tool use enables artifact delivery. Computer's expansion thus reflects not only a quantitative shift in cognitive demand and competency breadth but a qualitative one: at fine granularity, a meaningful fraction of Computer usage targets work that users essentially never directed at Search.

\begin{figure}[!htbp]
    \centering
    \includegraphics[width=0.7\textwidth]{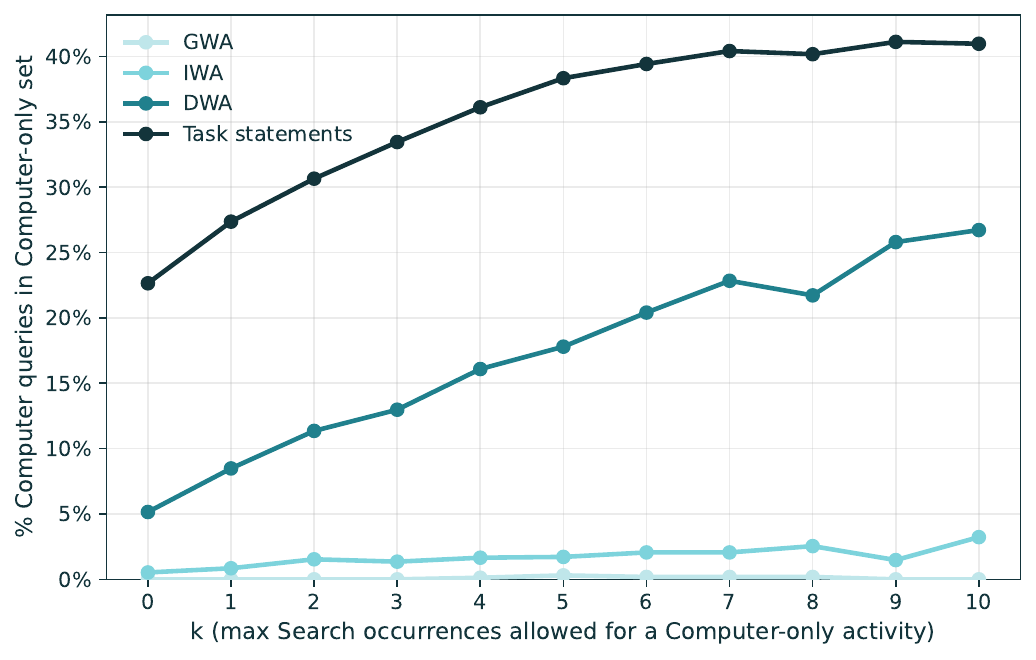}
    \caption{Share of Computer queries that engage at least one Computer-only activity, swept across the Search-occurrence ceiling $k$ and the four O*NET nesting levels. An activity is in the Computer-only set at threshold $k$ if it appears more than $k$ times in Computer queries and at most $k$ times in Search queries from the same dual-product users. At the strictest definition ($k=0$, Search occurrences must be zero), $23\%$ of Computer queries that engage at least one TS land on a Computer-only TS; at $k=5$ the share rises to $38\%$ and plateaus near $41\%$ by $k=7$. Shares are computed among Computer queries that engage at least one activity at the corresponding level. The GWA and IWA curves both flatten near zero, indicating that Search and Computer cover the same coarse topical surface; the TS curve grows steeply, indicating that Computer's distinctiveness lies in fine-grained executional work rather than topical range.}
    \label{fig:setdiff_curve}
\end{figure}

The Computer queries that unlock new tasks are systematically more complex than the rest of the Computer sample. Table~\ref{tab:unlocked_complexity} compares Computer queries that engage at least one Computer-only task statement (``unlocked'' queries, $k=0$) with other Computer queries that engage at least one task statement. Unlocked queries are more likely to require higher-order cognition ($81.9\%$ at Bloom's \emph{Analyze} level or above versus $75.4\%$) and abstract, non-routine work ($78.8\%$ versus $73.2\%$), draw on more knowledge domains ($2.67$ versus $2.50$), and decompose into more distinct work activities ($3.28$ versus $3.05$ GWAs, $4.46$ versus $4.21$ IWAs, $4.52$ versus $4.03$ DWAs, and $5.68$ versus $4.29$ TSs). This indicates that the work Computer uniquely unlocks skews toward more challenging tasks.

\begin{table}[!htbp]
\centering
\small
\caption{Complexity of unlocked Computer queries (engaging $\geq 1$ Computer-only task statement at $k=0$) versus other Computer queries that also engage $\geq 1$ task statement. The 864 Computer queries (of 5{,}000) that engage no task statement are excluded, so the two groups total 4{,}136. Shares are computed among queries classifiable on the relevant dimension.}
\label{tab:unlocked_complexity}
\begin{tabular}{lcc}
\toprule
& Unlocked queries & Other Computer queries \\
\midrule
\textit{Cognitive level} & & \\
\quad Higher-order cognition (Bloom $\geq$ Analyze) & $81.9\%$ & $75.4\%$ \\
\quad Abstract (non-routine) tasks & $78.8\%$ & $73.2\%$ \\
\addlinespace
\textit{Knowledge breadth} & & \\
\quad Mean Knowledge Domains & $2.67$ & $2.50$ \\
\addlinespace
\textit{Task composability} & & \\
\quad Mean Generalized Work Activities (GWA)  & $3.28$ & $3.05$ \\
\quad Mean Intermediate Work Activities (IWA)  & $4.46$ & $4.21$ \\
\quad Mean Detailed Work Activities (DWA)  & $4.52$ & $4.03$ \\
\quad Mean Task Statements (TS) & $5.68$ & $4.29$ \\
\midrule
$N$ & $937$ & $3{,}199$ \\
\bottomrule
\end{tabular}
\end{table}

\section{Discussion}\label{discussion}

This paper documents the downstream task-level economic implications of giving users access to autonomous task execution enabled by AI agents: Computer completes tasks autonomously and with higher quality (Section~\ref{autonomy}), which dramatically reduces the human time and cost required (Section~\ref{efficiency}) and shifts activity toward broader and more cognitively demanding work (Section~\ref{scope}). We make a few concluding comments and discuss the limitations and directions for future research.

\paragraph{The role of autonomy.} Autonomy is the key that unlocks other downstream effects. Agents eliminate the manual task-decomposition and execution loop in conversational sessions and produce higher-quality outputs. The pre-agent binding constraint on what users attempt is therefore not information access but execution capacity. When execution is delegated to the agent, the user's role shifts from operator to supervisor, reallocating time toward higher-order work such as direction, verification, and task extension.

\paragraph{From speed to scope.} The existing evidence has largely focused on productivity; our results suggest that a productivity framing might understate the impact of autonomous agents. Although the estimated time and cost savings are large, the more consequential finding may be scope expansion: users undertake tasks outside their primary domains and at higher levels of complexity.

\paragraph{Limitations.} We note a few important caveats. First, the 3-month observation window (Feb~27--May~27, 2026) captures an early-adoption period in which users are disproportionately power users and paying subscribers. Whether the patterns generalize to a broader population as the product matures remains an open question. Second, the matched-query methodology, while controlling for task content, applies only to the subset of Computer queries with close Search equivalents. As the scope analysis suggests, many Computer queries lack such analogues, and our estimates of autonomy and efficiency may not generalize to these tasks. Nevertheless, given that Computer queries tend to be more complex, the corresponding gains are likely to be larger for these tasks. Third, although sessions provide a natural unit for organizing distinct tasks, they are noisy proxies: users may distribute a single task across multiple sessions or, conversely, undertake multiple unrelated tasks within a single session. Fourth, the efficiency estimates rely on assumed per-tool human-equivalent times, human supervision time, and LLM estimates. Although the breakeven and sensitivity analysis shows robustness to mismeasurement, the absolute magnitudes should be interpreted as approximate. Fifth, the scope analysis relies heavily on LLM-based classification, which also introduces measurement error. However, the magnitude of the gaps suggests that the patterns are unlikely to be driven by classification noise alone. Lastly, we measure user behavior within the Perplexity ecosystem and do not observe activity outside it; as a result, we may not capture the full scope of users' workflow and tool use.

\paragraph{Future directions.} 
Our conceptual framework and empirical analysis are restricted to the individual-worker and task levels, and therefore leave open how firms, consumers, or the broader labor market may respond. A natural next step is to study how these micro-level changes aggregate to organizational and labor-market outcomes. If agents lower barriers to entry across occupational boundaries and expertise levels, and reduce coordination costs, individuals may produce outputs that previously required teams. The relevant margin then extends beyond productivity on existing tasks to the recomposition of job bundles and team structures. Future work could link agent usage to downstream workplace outcomes to assess whether agents are directed primarily toward accelerating existing workers, enabling workers to assume cross-occupational responsibilities, or creating new categories of economically viable work. Our early evidence points to all three, but firm-level production and employment data are needed to assess how agents ultimately reshape the bundling of work, the definition of roles, and the structure of teams.

\clearpage
\bibliographystyle{abbrvnat}
\bibliography{references}

\clearpage
\begin{appendices}

\section*{Appendix}

\section{Proofs for the Conceptual Framework}\label{app:proofs}

We collect here the proofs of all lemmas, propositions, and corollaries in Section~\ref{framework}.

\begin{proof}[\textbf{Proof of Lemma~\ref{lem:sorting} (Agent is preferred for more complex tasks)}]
The user strictly prefers the agent mode if and only if $C(s;\Agent) < C(s;\Conversational)$, i.e.\ $f_{\Agent} + m_{\Agent} s < f_{\Conversational} + m_{\Conversational} s$. Rearranging gives
$s > (f_{\Agent} - f_{\Conversational}) / (m_{\Conversational} - m_{\Agent}).
$
The numerator is positive by Assumption~\ref{asmp:fixed} and the denominator by Assumption~\ref{asmp:marginal}, so $s^{\ast}$ is positive and well-defined. For $s < s^{\ast}$ the inequality reverses, so the user strictly prefers the conversational mode; at $s = s^{\ast}$ the two modes cost the same and the user is indifferent.
\end{proof}

\begin{proof}[\textbf{Proof of Proposition~\ref{prop:afford} (Affordable value frontier expands)}]
By Assumption~\ref{asmp:completion} task opportunities are indivisible, so individual affordability under toolkit $\mathcal{T}$ reduces to the per-task condition $c_j^{\mathcal{T}}\leq B$. Agent access preserves the conversational mode as an option, so for every task $j$,
\[
c_j^{\text{post}}
=
\min\{C(s_j;\Conversational),C(s_j;\Agent)\}
\leq
C(s_j;\Conversational)
=
c_j^{\text{pre}}.
\]
Therefore $c_j^{\text{pre}}\leq B$ implies $c_j^{\text{post}}\leq B$. The pre-period individually affordable set is a subset of the post-period individually affordable set, so $u^{\text{post}}\geq u^{\text{pre}}$.

By the task primitives and Assumption~\ref{asmp:value}, tasks are weakly ordered by step count, and value weakly increases along this order. Moreover, $c_j^{\text{pre}}$ and $c_j^{\text{post}}$ are weakly increasing in $j$ (each is a minimum of costs that are increasing in $s_j$), so the individually affordable sets are prefixes $\{1,\ldots,u^{\text{pre}}\}$ and $\{1,\ldots,u^{\text{post}}\}$. If $u^{\text{post}}>u^{\text{pre}}$, then every $j\in\{u^{\text{pre}}+1,\ldots,u^{\text{post}}\}$ is unaffordable under the conversational mode and affordable under the post-agent toolkit by the definition of the endpoints; value monotonicity and the $v_0 \equiv 0$ convention (Assumption~\ref{asmp:value}) then give $v_{u^{\text{post}}}\geq v_{u^{\text{pre}}}$. If $u^{\text{post}}=u^{\text{pre}}$, the claim is immediate.
\end{proof}

\begin{proof}[\textbf{Proof of Proposition~\ref{prop:value} (Total value expands)}]
By Assumption~\ref{asmp:completion}, realized value $W^{\mathcal{T}} = \sum_j a^{\ast}_{\mathcal{T},j} v_j$ counts only fully completed tasks, so comparing toolkits reduces to comparing their feasible attempt policies. Let $a^{\ast}_{\text{pre},j}$ be the pre-period optimum, so
\[
\sum_{j=1}^{J} a^{\ast}_{\text{pre},j} c_j^{\text{pre}}\leq B.
\]
Because $c_j^{\text{post}}\leq c_j^{\text{pre}}$ for every task $j$, the same policy is weakly cheaper under post-period costs:
\[
\sum_{j=1}^{J} a^{\ast}_{\text{pre},j} c_j^{\text{post}}
\leq
\sum_{j=1}^{J} a^{\ast}_{\text{pre},j} c_j^{\text{pre}}
\leq B.
\]
The post-period optimizer can therefore imitate the pre-period policy. By optimality of $a^{\ast}_{\text{post}}$,
\[
W^{\text{post}} \;\geq\; W^{\text{pre}}.
\]
\end{proof}

\begin{proof}[\textbf{Proof of Corollary~\ref{cor:surplus} (Total surplus expands)}]
When the conversational-only chosen bundle exhausts the aggregate budget, $K^{\text{pre}} = B$, and the post-period policy is feasible by construction, so $K^{\text{post}} \leq B = K^{\text{pre}}$. Combined with $W^{\text{post}} \geq W^{\text{pre}}$ from Proposition~\ref{prop:value},
\[
\Pi^{\text{post}}-\Pi^{\text{pre}}
=
(W^{\text{post}}-W^{\text{pre}})-(K^{\text{post}}-K^{\text{pre}})
\geq 0.
\]
\end{proof}

\begin{proof}[\textbf{Proof of Corollary~\ref{cor:ratio} (Value-to-cost ratio expands)}]
When the conversational-only chosen bundle exhausts the aggregate budget, $K^{\text{pre}} = B > 0$, and the post-period budget constraint gives $K^{\text{post}} \leq B = K^{\text{pre}}$. The pre-period selected set is therefore nonempty, and since task values are positive (Assumption~\ref{asmp:value}), $W^{\text{pre}} > 0$; Proposition~\ref{prop:value} then gives $W^{\text{post}} \geq W^{\text{pre}} > 0$, so the post-period selected set is also nonempty and $K^{\text{post}} > 0$. Therefore
\[
\frac{W^{\text{post}}}{K^{\text{post}}} \;\geq\; \frac{W^{\text{pre}}}{K^{\text{post}}} \;\geq\; \frac{W^{\text{pre}}}{K^{\text{pre}}},
\]
where the first inequality uses $W^{\text{post}} \geq W^{\text{pre}}$ and the second uses $K^{\text{post}} \leq K^{\text{pre}}$.
\end{proof}

\begin{proof}[\textbf{Proof of Proposition~\ref{prop:decomp} (Surplus decomposition)}]
Starting from $\Pi^{\mathcal{T}}=W^{\mathcal{T}}-K^{\mathcal{T}}$,
\begin{align*}
\Delta\Pi
&=
\sum_{j\in A^{\text{post}}}\big(v_j-c_j^{\text{post}}\big)
-
\sum_{j\in A^{\text{pre}}}\big(v_j-c_j^{\text{pre}}\big) \\
&=
\sum_{j\in A^{\text{pre}}\cap A^{\text{post}}}
\big(c_j^{\text{pre}}-c_j^{\text{post}}\big)
+
\sum_{j\in A^{\text{post}}\setminus A^{\text{pre}}}
\big(v_j-c_j^{\text{post}}\big)
-
\sum_{j\in A^{\text{pre}}\setminus A^{\text{post}}}
\big(v_j-c_j^{\text{pre}}\big).
\end{align*}
The first term is weakly non-negative because $c_j^{\text{post}}\leq c_j^{\text{pre}}$ for every task. The entry and exit terms depend on selected task values and are not separately signed by these assumptions alone.
\end{proof}

\clearpage
\section{Numerical Example via Dynamic Programming}\label{app:dp}\label{app:numerical}

The discrete optimization problem in Section~\ref{framework},
\[
\max_{\{a_j\}_{j=1}^{J}} \sum_{j=1}^{J} a_j v_j
\qquad \text{s.t.} \qquad
\sum_{j=1}^{J} a_j c_j^{\mathcal{T}} \leq B,
\qquad a_j \in \{0, 1\},
\]
is a 0/1 knapsack problem and admits a standard dynamic programming solution. We describe the recurrence first and then illustrate it with a four-task example.

\subsection{Dynamic programming solution}

\paragraph{Recurrence.} Let $V_{\mathcal{T}}(j, b)$ denote the maximum total value attainable using only tasks $\{1, \ldots, j\}$ with remaining budget $b$, where $b$ runs over a discretization of $[0, B]$ fine enough to resolve the costs $\{c_j^{\mathcal{T}}\}$. The recurrence is
\[
V_{\mathcal{T}}(j, b)
\;=\;
\begin{cases}
\max\bigl\{\, V_{\mathcal{T}}(j-1, b),\;\; V_{\mathcal{T}}(j-1, b - c_j^{\mathcal{T}}) + v_j\,\bigr\}
& \text{if } c_j^{\mathcal{T}} \leq b, \\[4pt]
V_{\mathcal{T}}(j-1, b)
& \text{if } c_j^{\mathcal{T}} > b,
\end{cases}
\]
with boundary $V_{\mathcal{T}}(0, b) = 0$ for all $b$. The optimal value of the user's problem is $W^{\mathcal{T}} = V_{\mathcal{T}}(J, B)$. The two cases of the recurrence correspond to whether task $j$ is included: take it (gain $v_j$, lose $c_j^{\mathcal{T}}$ from the budget) or skip it. The maximum over the two encodes optimal substructure.

\paragraph{Recovering the optimal task set.} The selected set $A^{\mathcal{T}} = \{j : a^{\ast}_{\mathcal{T},j} = 1\}$ is recovered by backward tracing. Starting at $(j, b) = (J, B)$, set
\[
a^{\ast}_{\mathcal{T},j}
\;=\;
\begin{cases}
1 & \text{if } c_j^{\mathcal{T}} \leq b \text{ and } V_{\mathcal{T}}(j, b) = V_{\mathcal{T}}(j-1, b - c_j^{\mathcal{T}}) + v_j, \\
0 & \text{otherwise},
\end{cases}
\]
then update $b \leftarrow b - a^{\ast}_{\mathcal{T},j}\, c_j^{\mathcal{T}}$ and $j$. After $J$ iterations every $a^{\ast}_{\mathcal{T},j}$ is determined and the implied total cost is $K^{\mathcal{T}} = \sum_j a^{\ast}_{\mathcal{T},j}\, c_j^{\mathcal{T}}$.

\subsection{Application to a four-task example}

We apply the DP to a stylized four-task example chosen to illustrate the propositions. Tasks are indexed by $j \in \{1,2,3,4\}$ with step counts
\[
s_1 = 1,\qquad s_2 = 2,\qquad s_3 = 3,\qquad s_4 = 10.
\]
In this example, the conversational mode alone attempts the middle task set $\{2,3\}$, while the post-agent toolkit attempts the expanded set $\{1,2,3,4\}$. This pattern is illustrative, not a general implication of the assumptions alone.

\paragraph{Parameters.} Set $f_{\Conversational} = 1$, $m_{\Conversational} = 17$, $f_{\Agent} = 18$, $m_{\Agent} = 1$ (Assumptions~\ref{asmp:fixed} and~\ref{asmp:marginal}: agent has higher fixed cost but lower marginal cost). Task values are
\[
v_1=20,\qquad v_2=50,\qquad v_3=70,\qquad v_4=200,
\]
satisfying Assumption~\ref{asmp:value} strictly in this example. The aggregate budget is $B = 87$. The induced upper affordability endpoints are $u^{\text{pre}}=3$ and $u^{\text{post}}=4$.

The breakeven threshold from Lemma~\ref{lem:sorting} is
\[
s^{\ast} = \frac{f_{\Agent} - f_{\Conversational}}{m_{\Conversational} - m_{\Agent}} = \frac{18-1}{17-1} = \tfrac{17}{16} \approx 1.06.
\]
For $s > s^{\ast}$ agent is cheaper; for $s < s^{\ast}$ the conversational mode is cheaper.

\paragraph{Per-task costs.} The mode-specific and effective costs at each $s_j$ are:
\begin{center}
\begin{tabular}{cccccccc}
\toprule
$j$ & $s_j$ & $v_j$ & $C(s_j; \Conversational)$ & $C(s_j; \Agent)$ & $c_j^{\text{pre}}$ & $c_j^{\text{post}}$ & Mode chosen post \\
\midrule
1 & 1  & 20  & 18  & 19 & 18  & 18 & Conversational \\
2 & 2  & 50  & 35  & 20 & 35  & 20 & Agent \\
3 & 3  & 70  & 52  & 21 & 52  & 21 & Agent \\
4 & 10 & 200 & 171 & 28 & 171 & 28 & Agent \\
\bottomrule
\end{tabular}
\end{center}

\paragraph{Pre-period optimum (conversational only).} Apply the DP recurrence with $J=4$, $B=87$, and pre-period costs $\{18, 35, 52, 171\}$. Task~$4$ is individually unaffordable since $c_4^{\text{pre}}=171 > 87$, so the recurrence's second branch removes it from consideration. Among the remaining tasks, $\{1,2,3\}$ would cost $18+35+52 = 105 > 87$, while $\{2,3\}$ costs $35+52 = 87$. The DP yields $V_{\text{pre}}(4, 87) = 120$ with selected set $A^{\text{pre}} = \{2, 3\}$:
\[
K^{\text{pre}} = 87, \quad W^{\text{pre}} = 120, \quad \Pi^{\text{pre}} = 33.
\]

\paragraph{Post-period optimum (conversational + agent).} Apply the DP with post-period costs $\{18, 20, 21, 28\}$. Agent reduces the cost of tasks $\{2, 3\}$ from $35+52 = 87$ to $20+21 = 41$, freeing $46$ units of budget; that relief exactly pays for the two previously omitted endpoints (task $1$ at $18$ and task $4$ at $28$). All four tasks are jointly affordable:
\[
18 + 20 + 21 + 28 = 87 = B.
\]
The DP yields $V_{\text{post}}(4, 87) = 340$ with $A^{\text{post}} = \{1, 2, 3, 4\}$:
\[
K^{\text{post}} = 87, \quad W^{\text{post}} = 340, \quad \Pi^{\text{post}} = 253.
\]
The mode mix in the post period reflects the sorting threshold: the short task $j = 1$ stays on the conversational mode (cost $18 < 19$), while tasks $j \in \{2,3,4\}$ migrate to agent.

\paragraph{Illustrating the propositions.}

We use this example to illustrate our prior propositions:
\begin{itemize}
\item \emph{Lemma~\ref{lem:sorting} (sorting).} Realized post-period choices respect the threshold $s^{\ast} \approx 1.06$: $j = 1$ stays on the conversational mode, while $j \in \{2,3,4\}$ go to agent.
\item \emph{Proposition~\ref{prop:afford}.} The affordable value frontier expands: task $4$ has $c_4^{\text{pre}}=171>B=87\geq 28=c_4^{\text{post}}$, so $u^{\text{post}}=4>3=u^{\text{pre}}$ and $v_4\geq v_3$.
\item \emph{Proposition~\ref{prop:value}.} Total value rises: $W^{\text{post}} = 340 \geq 120 = W^{\text{pre}}$.
\item \emph{Corollary~\ref{cor:surplus}.} The corollary's condition holds: the pre-period budget is exhausted ($K^{\text{pre}} = B = 87$). Total surplus rises: $\Pi^{\text{post}} = 253 \geq 33 = \Pi^{\text{pre}}$ (here realized cost happens to be unchanged, $K^{\text{post}} = 87 = K^{\text{pre}}$).
\item \emph{Corollary~\ref{cor:ratio}.} The aggregate value-to-cost ratio rises: $W^{\text{post}}/K^{\text{post}} = 340/87 > 120/87 = W^{\text{pre}}/K^{\text{pre}}$.
\item \emph{Proposition~\ref{prop:decomp}.} The surplus gain decomposes into an intensive margin of $46$ from cost savings on retained tasks, $(35-20)+(52-21)$, an entry margin of $174$ from adding tasks $1$ and $4$, $(20-18)+(200-28)$, and no exit term; $46+174=220=\Delta\Pi$.
\end{itemize}

\clearpage
\section{Sensitivity Analysis}

The tool-based efficiency estimates (Section~\ref{efficiency}) rest on two assumptions about human time: the per-tool human-equivalent time estimates (Table~\ref{tab:tool_classification}), and the 10-minute oversight charged to Computer. We trace how Computer~+~Human's advantage erodes as we make each assumption increasingly conservative, for both cost (Section~\ref{app:cost_sensitivity}) and time (Section~\ref{app:time_sensitivity}).

In every figure, the $x$-axis is the stress factor, and each panel sweeps one assumption. The left panel deflates the per-tool human-equivalent time estimates by this factor (counterfactual human time = original / factor), decreasing the human time estimate for Search~+~Human. The right panel inflates the 10-minute oversight assumption (counterfactual human time = original $\times$ factor), increasing the human time estimate for Computer~+~Human. Each thin line is one of the 18 domains; the bold line is the sample-weighted overall.

\subsection{Cost advantage}\label{app:cost_sensitivity}

Figure~\ref{fig:cost_sensitivity} plots the cost saving relative to Search~+~Human. Computer retains a cost advantage at every level until each curve crosses zero: under the per-tool stress the overall advantage survives to $16\times$ and the tightest domain (Travel) to $8\times$; under the oversight stress, to $26\times$ overall and $12\times$ in the tightest domain (Consumer Goods).

\begin{figure}[!htbp]
    \centering
    \includegraphics[width=\textwidth]{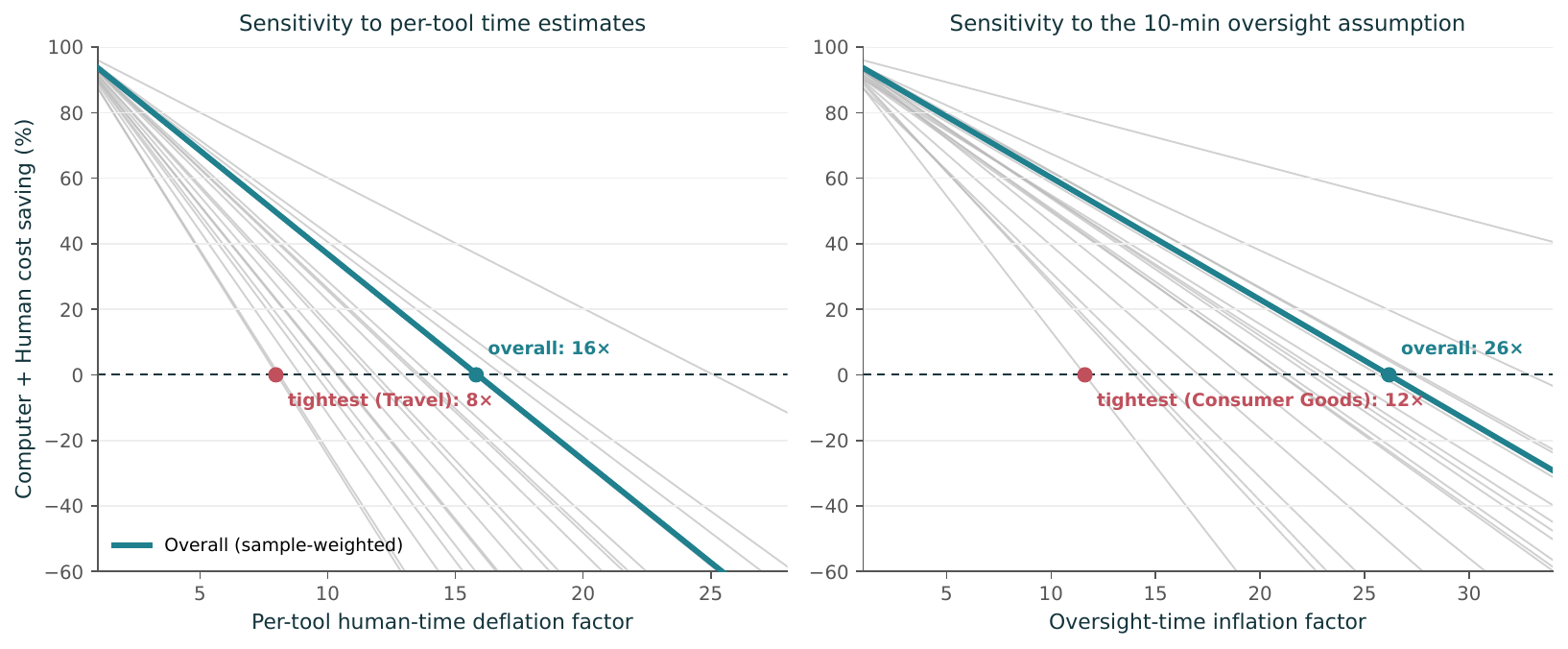}
    \caption{Computer~+~Human cost saving (\%) versus two stress factors on the human-time assumptions. \textbf{Left:}~Per-tool human-equivalent time estimates deflated by the factor on the $x$-axis. \textbf{Right:}~The 10-minute oversight assumption inflated by the factor on the $x$-axis. Thin gray lines are the 18 individual domains; the bold teal line is the sample-weighted overall. The dashed line marks zero cost saving (breakeven); markers indicate where the overall and the tightest (earliest-crossing) domain reach it.}
    \label{fig:cost_sensitivity}
\end{figure}

\subsection{Time advantage}\label{app:time_sensitivity}

Figure~\ref{fig:time_sensitivity} repeats the exercise in wall-clock minutes rather than dollars. Computer retains a time advantage at every level until each curve crosses zero: under the per-tool stress the overall advantage survives to $7\times$ and the tightest domain (Consumer Goods) to $5\times$; under the oversight stress, to $24\times$ overall and $11\times$ in the tightest domain (Consumer Goods).

\begin{figure}[!htbp]
    \centering
    \includegraphics[width=\textwidth]{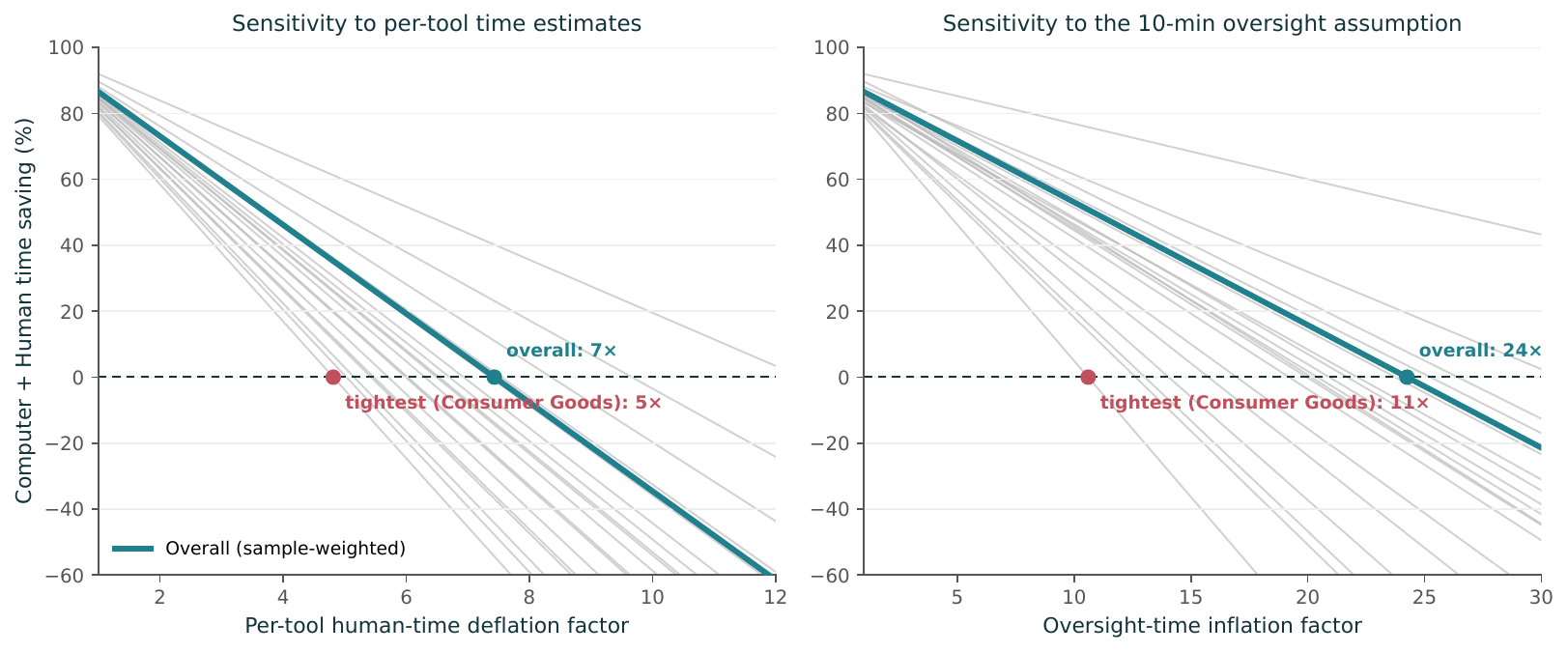}
    \caption{Computer~+~Human time saving (\%) versus two stress factors on the human-time assumptions. \textbf{Left:}~Per-tool human-equivalent time estimates deflated by the factor on the $x$-axis. \textbf{Right:}~The 10-minute oversight assumption inflated by the factor on the $x$-axis. Thin gray lines are the 18 individual domains; the bold teal line is the sample-weighted overall. The dashed line marks zero time saving (breakeven); markers indicate where the overall and the tightest (earliest-crossing) domain reach it.}
    \label{fig:time_sensitivity}
\end{figure}

\clearpage
\section{Cross-Product Complementarity}\label{app:complementarity}

A natural question at the product level is whether Computer queries are substitutes for or complements to Search queries. The conceptual framework suggests two countervailing forces. On one hand, users could substitute Computer for Search on longer tasks, reducing Search usage. On the other hand, Computer could free up budget to attempt additional shorter tasks where Search has the cost advantage, increasing Search usage. We therefore estimate the adoption effect directly. Although Search queries grew faster among Computer users than among non-Computer users (Figure~\ref{fig:adoption}), Computer users may have been more active to begin with.

To isolate the causal effect, we construct a matched sample using exact matching on three dimensions: subscription tier (Pro, Max, free), primary search topic (20~domain categories derived from the user's most frequent domain in the pre-period Search queries), and pre-period Search intensity (quartile bins). We draw a sample of 100{,}000 Computer adopters whose first use of Computer falls between February~27 and May~13, and restrict the sample to users with at least one Search query in the pre-adoption period, leaving 61{,}913 adopters. We then draw an approximately equal-sized control group of non-adopters from the same tier$\times$topic$\times$intensity cells. This design ensures that each adopter is compared to a non-adopter with the same subscription plan, the same primary search interest, and a similar baseline search rate.

We estimate a two-way fixed-effects (TWFE) difference-in-differences model on a balanced user-day panel spanning February~13 to May~27 (104~days):
\begin{equation}\label{eq:did}
    \text{SearchQueries}_{it} \;=\; \alpha_i \;+\; \gamma_t \;+\; \beta \cdot \text{Post}_{it} \;+\; \varepsilon_{it}
\end{equation}
where $\text{SearchQueries}_{it}$ is the number of Search queries by user~$i$ on day~$t$, $\alpha_i$ and $\gamma_t$ are user and day fixed effects, and $\text{Post}_{it}$ is an indicator equal to one if day~$t$ falls on or after user~$i$'s first Computer query (zero for never-adopted controls). Standard errors are clustered at the user level. A positive (negative) $\beta$ indicates complementarity (substitution) between Computer adoption and Search use. We also estimate an intensive-margin specification on the treated subsample:
\begin{equation}\label{eq:intensive}
    \text{SearchQueries}_{it} \;=\; \alpha_i \;+\; \gamma_t \;+\; \delta \cdot \text{ComputerQueries}_{it} \;+\; \varepsilon_{it}
\end{equation}
where $\text{ComputerQueries}_{it}$ is the number of Computer queries issued by user~$i$ on day~$t$. $\delta$ captures the effect of the quantity of Computer queries, not of binary adoption.

Table~\ref{tab:cannibalization} reports the regression estimates. In column~(1), the TWFE estimate of Computer adoption on daily Search queries is $\hat{\beta} = 1.05$ ($p < 0.01$): adopting Computer is associated with 1.05 additional Search queries per day. In column~(2), each additional Computer query on a given day is associated with $\hat{\delta} = 0.019$ additional Search queries ($p < 0.01$), a small but significant effect. A concern with the TWFE specification is that already-treated users may serve as implicit controls for later adopters, biasing the estimate when treatment effects vary across cohorts \citep{goodman-bacon2021}. Columns~(3) and~(4) of Table~\ref{tab:cannibalization} check the robustness of the result with alternative estimators. In column~(3), we estimate a stacked DiD in which each adoption cohort (early: Feb~27--Mar~24; mid: Mar~25--Apr~18; late: Apr~19--May~13) is compared only to the never-treated control group, with cohort-specific user and day fixed effects. In column~(4), we estimate separate DiD regressions for each cohort and aggregate via precision (inverse variance)-weighted averaging. The three approaches yield consistent estimates ($\hat{\beta} = 1.05$--$1.12$). The complementarity is also broadly uniform across topics: estimating Equation~\ref{eq:did} separately within each of the 20~primary search-topic categories yields a positive $\hat{\beta}$ in every category.

\begin{table}[!htbp]
    \centering
    \begin{threeparttable}
    \begin{tabular}{l cccc}
        \toprule
        & (1) & (2) & (3) & (4) \\
        & TWFE & Intensive & Stacked & Cohort-Agg. \\
        \midrule
        $\text{Post}_{it}$ & $1.053^{***}$ & & $1.096^{***}$ & $1.117^{***}$ \\
        & $(0.031)$ & & $(0.030)$ & $(0.030)$ \\[6pt]
        $\text{ComputerQueries}_{it}$ & & $0.019^{***}$ & & \\
        & & $(0.006)$ & & \\[6pt]
        \midrule
        User FE & Yes & Yes & Cohort $\times$ User & --- \\
        Day FE  & Yes & Yes & Cohort $\times$ Day & --- \\
        \midrule
        Unique users & 123{,}699 & 61{,}913 & 123{,}699 & 123{,}699 \\
        \bottomrule
    \end{tabular}
    \begin{tablenotes}[flushleft]
        \small
        \item \textit{Notes:} 61{,}913 Computer adopters exactly matched to 61{,}786 non-adopters on subscription tier, primary search topic (20~categories), and pre-period Search intensity (quartile bins). Panel: Feb~13 -- May~27, 2026 (104~days). Column~(2): treated users only, post-adoption days. Column~(3): stacked DiD with three cohorts (early, mid, late); each cohort paired with all never-treated controls. Column~(4): precision-weighted average of cohort-specific DiD estimates (early: $\hat{\beta} = 0.939$; mid: $1.116$; late: $1.278$). Standard errors clustered by user in parentheses. Significance: $^{*}\,p<0.10$; $^{**}\,p<0.05$; $^{***}\,p<0.01$.
    \end{tablenotes}
    \end{threeparttable}
    \caption{Effect of Computer adoption on daily Search queries. Column~(1): TWFE DiD (Equation~\ref{eq:did}). Column~(2): intensive margin (Equation~\ref{eq:intensive}). Columns~(3)--(4): robustness to staggered adoption.}
    \label{tab:cannibalization}
\end{table}

\clearpage
\section{Top Activities at Each O*NET Nesting Level}\label{app:top10_activities}

Tables~\ref{tab:top10_gwa}--\ref{tab:top10_tasks} report the ten most prevalent O*NET Generalized Work Activities (GWAs), Intermediate Work Activities (IWAs), Detailed Work Activities (DWAs), and Task Statements (TSs) by Computer prevalence in the scope sample, with Search prevalence shown alongside. Computer queries are restricted to those that invoked at least one execution (``do'') tool; Search queries are sampled from the same users. Prevalence is the share of queries that engage at least one instance of that activity, classified by an LLM against the O*NET catalog.

\begin{table}[H]
\centering
\small
\begin{tabular}{p{0.62\textwidth}rr}
\toprule
GWA & Computer (\%) & Search (\%) \\
\midrule
Documenting/Recording Information & 59.4 & 29.5 \\
Getting Information & 57.9 & 56.0 \\
Thinking Creatively & 40.7 & 16.9 \\
Analyzing Data or Information & 38.3 & 23.8 \\
Working with Computers & 14.4 & 7.3 \\
Communicating with People Outside the Organization & 11.4 & 17.5 \\
Monitoring Processes, Materials, or Surroundings & 7.0 & 5.9 \\
Estimating the Quantifiable Characteristics of Products, Events, or Information & 6.5 & 4.1 \\
Providing Consultation and Advice to Others & 6.5 & 11.9 \\
Processing Information & 6.4 & 2.9 \\
\bottomrule
\end{tabular}
\caption{Top 10 O*NET Generalized Work Activities (GWAs) by Computer prevalence, with Search side-by-side.}
\label{tab:top10_gwa}
\end{table}

\begin{table}[H]
\centering
\small
\begin{tabular}{p{0.62\textwidth}rr}
\toprule
IWA & Computer (\%) & Search (\%) \\
\midrule
Gather information from physical or electronic sources & 46.1 & 42.4 \\
Prepare informational or instructional materials & 29.5 & 13.8 \\
Create visual designs or displays & 24.8 & 6.9 \\
Prepare reports of operational or procedural activities & 19.2 & 5.9 \\
Analyze business or financial data & 16.9 & 8.4 \\
Read documents or materials to inform work processes & 15.1 & 12.2 \\
Prepare documentation for contracts, applications, or permits & 11.5 & 4.6 \\
Process digital or online data & 10.0 & 4.2 \\
Provide information or assistance to the public & 9.1 & 12.7 \\
Analyze market or industry conditions & 7.6 & 4.8 \\
\bottomrule
\end{tabular}
\caption{Top 10 O*NET Intermediate Work Activities (IWAs) by Computer prevalence, with Search side-by-side.}
\label{tab:top10_iwa}
\end{table}

\begin{table}[H]
\centering
\small
\begin{tabular}{p{0.62\textwidth}rr}
\toprule
DWA & Computer (\%) & Search (\%) \\
\midrule
Retrieve information from electronic sources & 27.7 & 19.7 \\
Search information sources to find specific data & 23.8 & 18.5 \\
Write informational material & 15.9 & 7.0 \\
Organize informational materials & 10.4 & 3.8 \\
Prepare graphics or other visual representations of information & 9.7 & 2.2 \\
Create computer-generated graphics or animation & 7.0 & 1.9 \\
Analyze financial information & 6.3 & 3.3 \\
Analyze business or financial data & 6.3 & 3.5 \\
Search files, databases or reference materials to obtain needed information & 6.3 & 4.1 \\
Prepare informational or reference materials & 6.2 & 1.9 \\
\bottomrule
\end{tabular}
\caption{Top 10 O*NET Detailed Work Activities (DWAs) by Computer prevalence, with Search side-by-side.}
\label{tab:top10_dwa}
\end{table}

\begin{table}[H]
\centering
\small
\begin{tabular}{p{0.62\textwidth}rr}
\toprule
TS & Computer (\%) & Search (\%) \\
\midrule
Search electronic sources, such as databases or repositories, or manual sources for information & 16.5 & 10.2 \\
Search standard reference materials, including online sources and the Internet, to answer pa\ldots & 10.0 & 10.3 \\
Conduct reference searches, using printed materials and in-house and online databases & 7.2 & 7.5 \\
Locate unusual or unique information in response to specific requests & 7.2 & 7.9 \\
Write text, such as stories, articles, editorials, or newsletters & 6.4 & 3.5 \\
Create graphs, charts, or other visualizations to convey the results of data analysis using \ldots & 4.6 & 0.5 \\
Report results of statistical analyses, including information in the form of graphs, charts,\ldots & 4.5 & 0.5 \\
Conduct internet-based and library research & 3.6 & 2.1 \\
Analyze business operations, trends, costs, revenues, financial commitments, and obligations\ldots & 3.0 & 1.5 \\
Retrieve electronic assets from repository for distribution to users, collecting and returni\ldots & 3.0 & 1.1 \\
\bottomrule
\end{tabular}
\caption{Top 10 O*NET Task Statements (TSs) by Computer prevalence, with Search side-by-side.}
\label{tab:top10_tasks}
\end{table}

\clearpage
\section{User Interview Highlights}\label{app:highlights}

We conducted 45-minute semi-structured interviews with 25 Computer users (6~enterprise users, 19~consumers), recruited from active Computer users with at least 5~queries. Participants walked through specific completed tasks, described their pre-Computer workflow for each, and estimated the time and cost that workflow would have taken. We group the findings into five recurring themes, summarizing the patterns across participants rather than attributing claims to individuals. These self-reports are subject to recall and selection bias and should be read as corroborative evidence; where participants gave concrete figures, we report them.

Section~\ref{app:highlights_quality} gathers \emph{quality} reports---self-reports of output quality that complement the next-turn user dissatisfaction signal in Section~\ref{autonomy}. Section~\ref{app:highlights_efficiency} collects \emph{efficiency} reports with specific before/after time and cost comparisons, which we use to cross-validate the efficiency estimates in Section~\ref{efficiency}. Section~\ref{app:highlights_recurrent} covers \emph{recurrent task automation}---automated workflows that users run repeatedly on a fixed schedule. Section~\ref{app:highlights_parallel} reports \emph{parallel work}, where asynchronous delegation lets users fire off tasks and continue other activity concurrently. Finally, Section~\ref{app:highlights_scope} documents \emph{scope expansion}---users performing tasks outside their domains of expertise---corroborating the scope analysis in Section~\ref{scope}.

\subsection{Quality improvement}\label{app:highlights_quality}

Participants across legal, founder, and product/consulting roles reported high output quality on demanding tasks. Legal users described high-quality legal first drafts and output on their hardest accounting task; founders reported technical specification documents praised by their technical leadership; and product/consulting users described polished, presentation-ready slide decks. These self-reports align with the low next-turn dissatisfaction observed in Section~\ref{autonomy}.

\subsection{Efficiency gains}\label{app:highlights_efficiency}

Efficiency claims were the most common and most quantified theme, with reported speedups spanning roughly $5\times$ to over $300\times$ (per-participant median $\approx$25$\times$) and cost reductions of two to three orders of magnitude. The gains clustered by work type:
\begin{itemize}
  \item \textbf{Legal work} clustered at the high end for drafting: legal research and writing dropped from $\approx$2~hours to $\approx$5~minutes ($\approx$24$\times$), drafting from 1--2~days to 5--10~minutes ($\approx$95$\times$), guardianship accounting from a full day to $\approx$30~minutes ($\approx$16$\times$), and an appraisal comparison report from $\approx$1~week to $\approx$15~minutes ($\approx$160$\times$).
  \item \textbf{Financial and consulting work} showed the largest synthesis gains: a risk-reporting framework went from $\approx$1~week to under an hour ($\approx$40$\times$), a 1{,}000-page synthesis from $\approx$2~weeks to $\approx$15~minutes ($\approx$320$\times$), and three major analyses from days--weeks to $\approx$1~hour ($\approx$10--80$\times$).
  \item \textbf{Product builds and websites} compressed from months to days: a full web product build from 3--6~months to 4--5~days ($\approx$15--35$\times$), a market-ready product delivered with $\approx$12~hours of human input over a 72-hour window ($\approx$6$\times$ leverage), and a full website rebuilt ground-up in $\approx$2.5~days. Course content dropped from 1--3~weeks to $\approx$1~day ($\approx$5--15$\times$).
  \item \textbf{Cost displacement} was largest on work that would otherwise require outside specialists or vendors, where reported spend fell by roughly two to three orders of magnitude ($\approx$120--750$\times$ cheaper).
\end{itemize}
These accounts cross-validate the production-data efficiency estimates in Section~\ref{efficiency}.

\subsection{Recurrent task automation}\label{app:highlights_recurrent}

Several users described moving from one-off prompting to scheduled, recurring automations. Reported workflows included periodic leadership reports assembled from cloud documents and posted to a review channel, weekend jobs that summarize a week's email, calendar, and project-tracker activity and propose the next week's priorities, and automated CRM synchronization and lead creation. A recurring benefit was keeping previously neglected workstreams---email backlogs, receivables collection---continuously up to date rather than handling them in batches.

\subsection{Parallel work}\label{app:highlights_parallel}

Because Computer executes autonomously, users described delegating tasks and continuing other work in parallel: firing off batches of asynchronous work and checking back later, setting up multi-step projects unattended, and submitting tasks before stepping away. Several characterized the resulting workflow as among the most productive stretches of their careers.

\subsection{Scope expansion}\label{app:highlights_scope}

Users repeatedly described taking on work outside their domain of expertise: handling software and data-engineering tasks, building software modules without a technical background, and preparing legal work without formal legal training. A recurring theme was spanning functions that normally require separate specialists, such as combining accounting, legal, and compliance work. These accounts corroborate the cross-occupation findings in Section~\ref{scope}.

\end{appendices}

\end{document}